\documentclass{article}

\PassOptionsToPackage{numbers, compress}{natbib}

\usepackage[preprint]{neurips_2021}

\usepackage[utf8]{inputenc} 
\usepackage[T1]{fontenc}    
\usepackage[pagebackref,breaklinks,colorlinks]{hyperref}
\usepackage{url}            
\usepackage{booktabs}       
\usepackage{amsfonts}       
\usepackage{nicefrac}       
\usepackage{microtype}      
\usepackage{xcolor}         
\usepackage{bm}

\usepackage{amsthm,paralist}
\usepackage{graphicx}
\usepackage{relsize,comment}
\usepackage{wrapfig}
\usepackage{algorithm}
\usepackage{algorithmic}
\usepackage{booktabs}
\usepackage{epsfig}
\usepackage{bbm}
\usepackage{relsize}
\usepackage{graphicx}
\usepackage[inline]{enumitem}
\usepackage{array}
\usepackage{subcaption}
\usepackage{booktabs,siunitx,caption}
\usepackage{mathrsfs}
\usepackage{colortbl}
\usepackage[usestackEOL]{stackengine}
\usepackage[most]{tcolorbox}
\usepackage{multicol}
\usepackage{mathtools}

\usepackage{amsmath}
\DeclareMathOperator*{\argmax}{arg\,max}

\newtheorem{theorem}{Theorem}
\newtheorem{example}[theorem]{Example}
\newtheorem{lemma}[theorem]{Lemma}

\newtheorem{remark}[theorem]{Remark}

\newtheorem{obsn}[theorem]{Observation}

\newcommand{\worstoffdro}{{Worst-off DRO}}
\newcommand{\groupdro}{{Group DRO}}
\newcommand{\unsupdro}{{Unsup DRO}}

\title{Towards Group Robustness in the presence of\\ Partial Group Labels}

\author{ Vishnu Suresh Lokhande\thanks{Work done as a research intern at Google.} \hspace{0.01in}\textsuperscript{\rm 1}, Kihyuk Sohn\textsuperscript{\rm 3}, Jinsung Yoon\textsuperscript{\rm 3}, Madeleine Udell\thanks{Work done as a visiting researcher at Google.} \hspace{0.02in}\textsuperscript{\rm 2},\\ \hspace{0.01in} \textbf{Chen-Yu Lee}\textsuperscript{\rm 3} \textbf{ and Tomas Pfister}\textsuperscript{\rm 3}\\
\textsuperscript{\rm 1}University of Wisconsin-Madison, \textsuperscript{\rm 2}Cornell University,
\textsuperscript{\rm 3}Google Cloud AI Research\\  
}

\begin{document}

\maketitle

\begin{abstract}
    Learning invariant representations is an important requirement when training machine learning models that are driven by spurious correlations in the datasets. 
    These spurious correlations, between input samples and the target labels, wrongly direct the neural network predictions resulting in poor performance on certain groups, especially the minority groups. 
    Robust training against these spurious correlations requires the knowledge of group membership for every sample. 
    Such a requirement is impractical in situations where the data labelling efforts for minority or rare groups is significantly laborious or where the individuals comprising the dataset choose to conceal sensitive information. 
    On the other hand, the presence of such data collection efforts result in datasets that contain partially labelled group information.
    Recent works have tackled the fully unsupervised scenario where no labels for groups are available. 
    Thus, we aim to fill the missing gap in the literature by tackling a more realistic setting that can leverage partially available sensitive or group information during training. 
    First, we construct a constraint set and derive a high probability bound for the group assignment to belong to the set.
    Second, we propose an algorithm that optimizes for the worst-off group assignments from the constraint set. 
    Through experiments on image and tabular datasets, we show improvements in the minority group's performance while preserving overall aggregate accuracy across groups.
\end{abstract}

\section{Introduction}
\label{sec:intro}
\vspace{-0.05in}

Neural networks being overly biased to certain groups of the data is an increasing concern within the machine learning community \citep{agarwal2018reductions}. 
A primary cause for bias against specific groups is the presence of extraneous attributes in the datasets that wrongly direct the model responses \citep{xie2017controllable}. 
Such extraneous attributes are features that need to be controlled for. 
For example, in computer vision tasks such as image classification or object detection, an extraneous attribute could correspond to the background in an image or a co-occurring object irrelevant to the task, e.g. a person making a speech in a football field could be predicted as playing football \citep{choi2019can}. 
The presence of such extraneous attributes warrant a model to derive the predictions by making spurious correlations to extraneous features in an image rather than an actual object of interest. 
An inevitable consequence of such correlations to extraneous attributes is disparities in performance across different groups within the dataset. 
Specifically, if certain groups form a minority, a model can simply \textit{cheat} by having a high overall aggregate accuracy but poor minority group accuracy \citep{oakden2020hidden}.

\begin{figure}
    \centering
    \begin{subfigure}{0.45\textwidth}
    \includegraphics[width=0.95\textwidth]{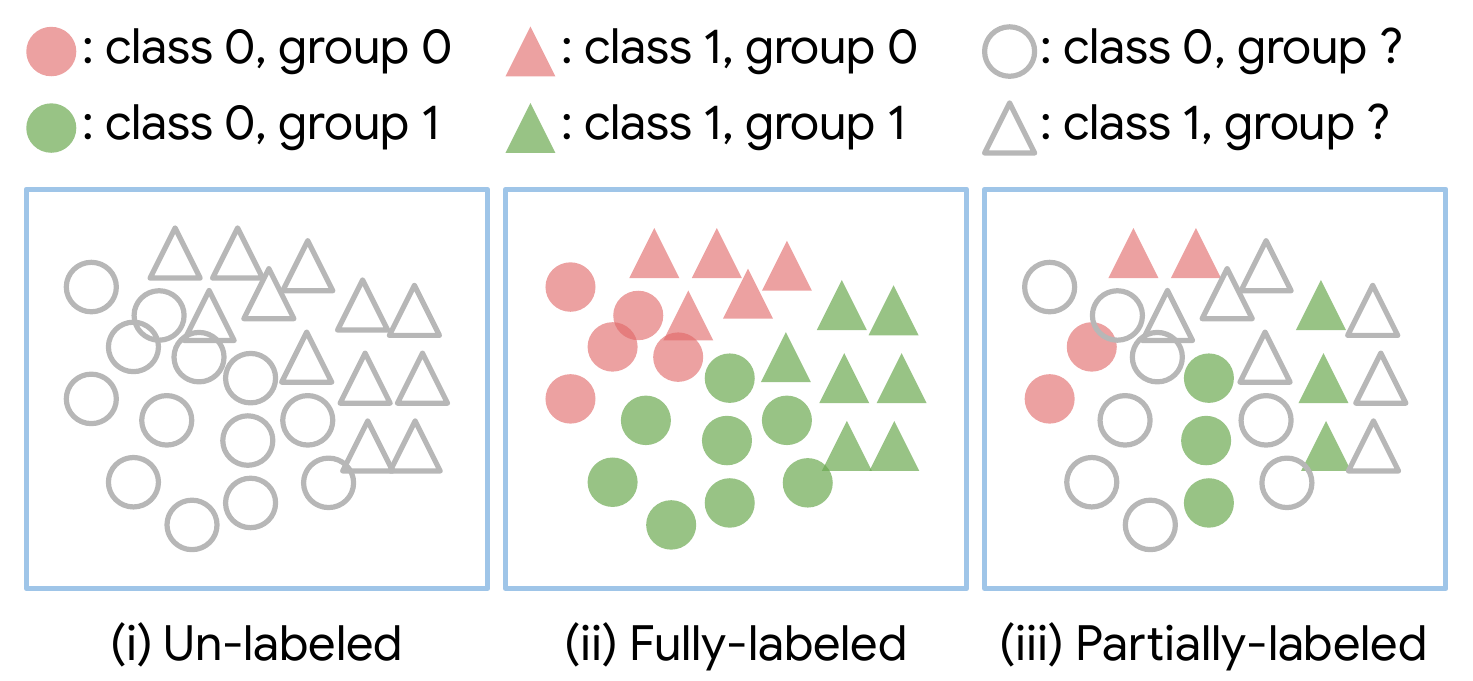}
    \caption{Problem settings.}
    \label{fig:teasar_setting}
    \end{subfigure}
    \hspace{0.05in}
    \begin{subfigure}{0.45\textwidth}
    \includegraphics[width=0.95\textwidth]{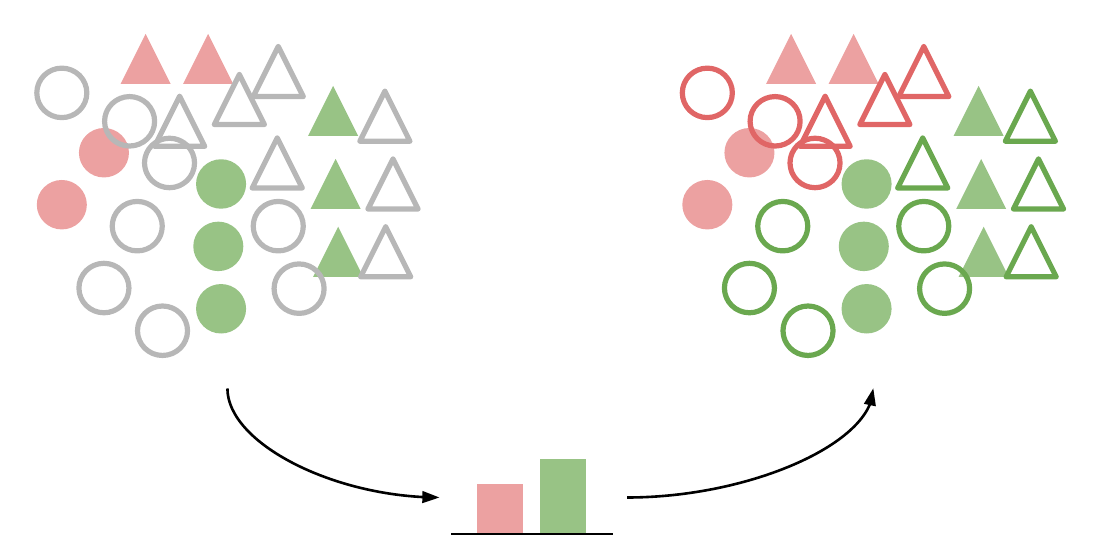}
    \caption{Find the worst-case group assignment with the marginal distribution constraint.}
    \label{fig:teasar_worstoffdro}
    \end{subfigure}
    \vspace{-0.05in}
    \caption{\footnotesize 
    (a) We introduce {\worstoffdro}, an invariant learning algorithm for  partial group labeled data (as shown in (iii)). This is in contrast with other settings, such as (unsupervised) DRO~\citep{hashimoto2018fairness} where (i) no group labels are available at train time, or {\groupdro}~\citep{sagawa2019distributionally} when (ii) group labels for all training examples are required. 
    (b) {\worstoffdro} finds the worst-case group assignment for unlabelled data using marginal distribution constraints, which may be given as side information, or estimated from labeled counterparts. 
    The constraint set defined by the marginal distribution includes the ground-truth group labels with high probability, ensuring the training objective is an upper bound of that of the {\groupdro} objective.}
    \label{fig:teasar}
    \vspace{-0.25in}
\end{figure}

Existing works for this problem~\citep{arjovsky2019invariant,sagawa2019distributionally} operate in the regime where the number of groups, likely to be adversely impacted through spurious correlations, are known apriori. 
Further, they assume a complete knowledge of the group membership of individual samples in the training dataset. 
While these methods have been proven effective, it is not realistic to assume access to the group labels for every sample. 
Consider the scenario where the minority / majority groups could be defined by demographic information such as \textit{gender} or \textit{race}. 
An individual can simply choose not to reveal this information due to privacy considerations \citep{kehrenberg2020null}. 
Alternatively, in medical image settings a label class could contain unrecognized subgroups that demand significant burden on the data labelling efforts \citep{NEURIPS2020_e0688d13}. 
An example of such unrecognized subgroups could be a lung cancer detection problem, where the class label could comprise of groups such as solid/subsolid tumors and central/peripheral neoplasms \citep{oakden2020hidden}, with many of these groups naturally forming a minority in the dataset. 
In this work, we consider a setting where a significant portion of the training data is devoid of group labels. 
We choose to fill a missing gap in the literature where several works bifurcate into methods that either are fully supervised or fully unsupervised in terms of the groups labels. 
Knowledge on the number of groups in the data makes it convenient to obtain group labels for a tiny portion of the data or take advantage of an existing labelled samples if available. 
Hence, we address the following research question: \textit{Can we train a model that is invariant to group membership when partially labelled data is available?}

We answer the question using a framework of distributionally robust optimization (DRO)~\citep{shapiro2021lectures,namkoong2016stochastic}. 
DRO allows for a training routine that optimizes for the worst-case training loss over predefined set of groups closely connected to the Rawlsian fairness measure \citep{rawls2001justice}. 
When the group membership is fully known, the method simply upweights/downweights average training loss of different groups through the course of training \citep{sagawa2019distributionally}. 
The application of DRO to the partial group label setting poses significant challenges: (1) the lack of group label makes it infeasible to compute the worst-off group loss; (2) optimizing only for the high-loss samples, by considering them as a worst-off group, discards considerable portion of the training data adversely impacting the overall accuracy; and (3) inferring missing group labels with pseudo-label based methods is a cause for ethical concerns especially when group labels are characteristic of sensitive information.

The third challenge above alludes a straightforward way of handling partially labelled setting where we directly estimate the group label for each sample with pseudo-label based methods. However, this approach could be harmful in the context of fairness problems because the estimated labels could be misused by a wrongdoer. For example, when the groups are indicative of sensitive information such as age or gender, an incorrect group estimation would wrongly designate the demographics of an individual. Moreover, when it's desirable and intended to conceal such sensitive information, a direct estimation of groups would be a violation of privacy. Thus we \textbf{cautiously avoid} building of utilizing pseudo-label based methods in this paper.

In light of all these challenges, we make the following \textbf{contributions}. We propose a method that defines a constraint set of group assignments and optimizes over all possibles configurations of the assignments within the set. Such a constraint set can encompass the group labels of the unlabeled portion of the data with high probability. We observe that optimizing for the a worst-off soft group assignment from the constraint set upper bounds the DRO objective with true group membership. Since worst-off assignments do not directly estimate the ground-truth assignments, our approach is theoretically safe and does not violate privacy. As we shall see in the paper, our method assigns high loss samples to groups with a higher weight and lower marginal probability improving the worst-off group accuracies. 
Further, the low loss samples are not discarded thus preserving the aggregate accuracy across groups. 
We show experiments on three imaging datasets and one tabular dataset and contrast the performance against several baselines.

\vspace{-0.15in}
\section{Related work}
\label{sec:related}
\vspace{-0.05in}

Distributionally robust optimization~\citep{ben2013robust,namkoong2016stochastic,duchi2021statistics} has been studied as a way to train robust ML models across multiple groups and environments. 

\noindent\textbf{Group Robust Optimization.} Methods in the literature handling robustness to extraneous attributes can be broadly categorized into two classes. The first class, domain generalization methods \cite{arjovsky2019invariant,mahajan2021domain,NEURIPS2018_invrep_noadv}, aim at learning representations invariant to a predefined set of extraneous attributes or groups. The goal is to be able to generalize to unseen domains or environments in the testing phase. On the other hand, the second class of methods, called as the group robust methods \citep{oakden2020hidden,sagawa2019distributionally,liu2021just}, seek to improve the worst-off or the minority group performance within the set of pre-defined groups. Here the training and testing phases share the same set of groups. Our approach falls into the second class of methods. 

\noindent\textbf{Robust Optimization with Demographics.} When group information is known at train time, {\groupdro}~\citep{hu2018does,oren2019distributionally,sagawa2019distributionally,mohri2019agnostic,zhang2020coping} or Invariant Risk Minimization (IRM)~\citep{arjovsky2019invariant} could be employed to improve the performance over multiple groups. Specifically, {\groupdro} proceeds by minimizing the loss of the group with the largest loss, while IRM enforces a shared predictor across multiple environments to be optimal in the form of a Lagrangian multiplier.

\noindent\textbf{Robust Optimization without Demographics.} As the group information may not be always available reliably, several studies have been focusing on developing methods that remove or reduce their dependence on the group information. \cite{hashimoto2018fairness} has developed a method based on the distributionally robust optimization that minimizes the loss of the samples with losses larger than a certain threshold. \cite{NEURIPS2020_07fc15c9} has proposed to reweight the samples in an adversarial way so that the high loss sample could receive more weight over the course of training. Moreover, \cite{liu2021just} has proposed a simple yet effective two-stage approach called Just-Train-Twice (JTT) that trains a model by upweighting samples with high losses from the initial ERM model. 

\noindent\textbf{Two-stage methods.}
\label{sec:twostage}
Recent methods, like JTT \cite{liu2021just} and EIIL \cite{creager2021environment},  which do not require group or demographic information adopt a two-stage pipeline for training. While the two-stage methods may attain better accuracy measures, they bear a few drawbacks in relation to a single-stage method. Firstly, two-stage methods introduce additional set of hyper-parameters that need to be tuned. For example, it's crucial for JTT to tune for the number of epochs to convergence in its first (identification) stage. Similarly, EIIL requires a pre-trained reference model and optimization of the EI objective that introduce several hyper-parameters. While certain parameters like learning rate, weight-decay and architecture could be shared across stages, such choice still need to be empirically verified on every new dataset. Secondly, in a two-stage model, a failed first stage leads to an unsuccessful second stage as errors from the first stage are propagated to the later stages. A first stage model could fail due to the model overfitting to the training data in the JTT method~\cite{liu2021just}, and similarly in EIIL inaccurate group inference may block second-stage invariant learning besides raising ethical issues on pseudo-label misuse. In summary, efforts to reduce a two-stage model to a single-stage method are beneficial and, as we shall see shortly, our proposal benefits from being a single stage method.


\vspace{-0.1in}
\section{Methodology}
\label{sec:method}
\vspace{-0.05in}
We introduce our robust optimization framework, {\worstoffdro}, with partial group information. We revisit the GroupDRO in Section~\ref{sec:method_problem} and detail our method in Section~\ref{sec:method_worstoffdro} and \ref{sec:method_marginal_distribution}. In Section~\ref{sec:method_optimization}, we describe a practical method for optimization.


\vspace{-0.05in}
\subsection{Preliminary: Group DRO}
\label{sec:method_problem}

Let $x\,{\in}\,\mathcal{X}\,{\subset}\,\mathbb{R}^{d}$ be data descriptors, $y\,{\in}\,\mathcal{Y}\,{\subset}\,\{0,1\}$ be target labels, and $g\,{\in}\,\mathcal{G}\,{\subset}\,\{1,...,M\}$ be group labels.
We assume training a neural network parameterized by the weights $w$ that corresponds to a per-sample loss $l(x, y; w)$.
Given data triplets $\{(x_{i},y_{i},g_{i})\}_{i=1}^{N}$, we seek to optimize $w$ for the Rawlsian criterion~\citep{rawls2001justice,zhang2014fairness,hashimoto2018fairness}, which minimizes the loss of the worst-off group, as follows:
%
\begin{equation}
    \min_{w} \max_{g\in\mathcal{G}} \mathbb{E}\big[l(x,y;w)|g\big].
    \label{eq:rawlsian}
\end{equation}
\citet{sagawa2019distributionally} proposed a practical algorithm to solve \eqref{eq:rawlsian}, called \textit{{\groupdro}}. This method optimizes a weighted expected loss across all groups. These weights over the groups, denoted by $q$, are drawn from a simplex $\Delta^M$. The objective function is as described below,
%
\begin{equation}
    \mathcal{L}_{\mathrm{GDRO}}=\min_{w}\max_{q\in\Delta^{M}} \sum_{j=1}^{M} \overbracket[1pt]{q_{j}}^{\parbox{2em}{\tiny Group\\Weights}} \underbracket[1pt]{\Big[\frac{\sum_{i=1}^{N} \overbracket[1pt]{\mathbbm{1}\{g_{i}=j\}}^{\text{Indicator function}} l(x_{i}, y_{i}, w)}{\sum_{i=1}^{N} \mathbbm{1}\{g_{i}=j\}}\Big]}_{\text{Per-group average loss}}\label{eq:groupdro_reform}
\end{equation}

\vspace{-0.1in}
\subsection{{\worstoffdro}}
\label{sec:method_worstoffdro}
\vspace{-0.05in}
In this work, we are interested in training a distributionally robust neural network when group labels are only partially available in the entire dataset. That is, our training dataset constitutes of the fully-labeled dataset $\{(x_{i}, y_{i}, g_{i}^{\star})\}_{i=1}^{K}$ and the task-labeled dataset $\{(x_{i}, y_{i}, -)\}_{i=K+1}^{N}$, where $-$ indicates the missing group labels.

As noted in \eqref{eq:groupdro_reform}, the {\groupdro} requires group labels of entire dataset. When some of them are missing, we propose to optimize for the following objective:
\begin{equation}
    \mathcal{L}_{\mathrm{WDRO}}(\mathcal{C})=\min_{w}\max_{q\in\Delta^{M}}\max_{\{\hat{g}\}\in\mathcal{C}}\sum_{j=1}^{M} q_{j} \Big[\frac{\sum_{i=1}^{N} \mathbbm{1}\{\hat{g}_{i}=j\} l(x_{i}, y_{i}, w)}{\sum_{i=1}^{N} \mathbbm{1}\{\hat{g}_{i}=j\}}\Big]\label{eq:worstoffdro_nocon}
\end{equation}
where $\mathcal{C}$ is a set of group assignments $\{\hat{g}_{i}\}_{i=1}^{N}$ satisfying $\hat{g}_{i}\,{=}\,g_{i}^{\star}, \forall i\leq K$. We call the objective in \eqref{eq:worstoffdro_nocon} a \textbf{\worstoffdro} as it optimizes neural network parameters with respect to the worst-off group assignment in a certain constraint set $\mathcal{C}$ (more details on the design of $\mathcal{C}$ soon). 

%
Note that the {\worstoffdro} objective forms an upper bound to the {\groupdro} objective evaluated at the ground-truth group labels if $\{g_{i}^{\star}\}_{i=1}^{N}\,{\in}\,\mathcal{C}$. Under identical parameters $w$ and $q$, this is rather a straightforward consequence from the fact that the ground-truth group assignment $\{g^{\star}_{i}\}_{i=1}^{N}$ falls within the constraint set $\mathcal{C}$. However, the following lemma generalizes the upper bound relationship between {\worstoffdro} and {\groupdro} objectives for all $w$ and $q$.

\begin{lemma}
    \label{lemma:upper_bound}
    Denote $\mathcal{L}_{\mathrm{GDRO}}$ at a given $w$ and $q$ parameters as $\mathcal{L}_{\mathrm{GDRO}(w, q)}$. Similarly $\mathcal{L}_{\mathrm{WDRO}}(\mathcal{C})$ at a fixed $w$ and $q$ as $\mathcal{L}_{\mathrm{WDRO}(w, q)}(\mathcal{C})$. When the ground-truth group assignment $\{g^{\star}_{i}\}_{i=1}^{N} \in \mathcal{C}$, we have
    \begin{align}
        \min_{w}\max_{q\in\Delta^{M}}\mathcal{L}_{\mathrm{GDRO}(w, q)} \le \min_{w}\max_{q\in\Delta^{M}}\mathcal{L}_{\mathrm{WDRO}(w, q)}(\mathcal{C})
    \end{align}
\end{lemma}
The proof is in Appendix~\ref{sec:proof_of_lemma}. For safety-critical applications, such as learning a fair classifier, it is important that the optimal objective (i.e., {\groupdro} with a ground-truth group assignment) is bounded by the objective used in optimization as in Lemma~\ref{lemma:upper_bound}. This is simply because optimizing the proposed learning objective \textbf{guarantees} that the corresponding lower bound of ground-truth {\groupdro} is also optimized. Conversely, objectives of methods such as EIIL~\citep{creager2021environment} or GEORGE~\citep{NEURIPS2020_e0688d13}, which optimize the {\groupdro} or IRM objectives using pseudo group labels, would not be comparable with $\mathcal{L}_{\mathrm{GDRO}}$ as they depend on a heuristic to obtain a single set of pseudo group labels.

\vspace{-0.05in}
\subsection{Reducing Constraint Set with Marginal Distribution Constraint}
\label{sec:method_marginal_distribution}
\vspace{-0.05in}
It is clear that the constraint set $\mathcal{C}$ plays an important role that connects {\worstoffdro} to {\groupdro}. Specifically, the {\worstoffdro} objective can be made a tighter bound to that of {\groupdro} by further constraining $\mathcal{C}$ so long as it contains the ground-truth group assignment $\{g_{i}^{\star}\}_{i=1}^{N}$. 
In the subsequent paragraph, we describe how we reduce the constraint set while retaining the ground-truth group assignment using a marginal distribution constraint. These constraints may be given as a side information or could be estimated from the small set of partial group labels under certain conditions.

Let $\mathcal{C}_{\mathbf{p}, \epsilon}$ is a subset of $\mathcal{C}$ whose elements $\{g_{i}\}_{i=1}^{N}$ satisfy the following condition:
\begin{gather}
g_{i}\,{=}\,g_{i}^{\star}, \forall i \leq K,\label{eq:constraint_labeled_data}\\
\vert\frac{1}{N}\sum\nolimits_{i=1}^{N} \mathbbm{1}\{g_{i}\,{=}\,j\}- \mathbf{p}_{j}\vert \leq \epsilon, \forall j \leq M,\label{eq:constraint_marginal}
\end{gather}
where \eqref{eq:constraint_labeled_data} implies that the true group labels are assigned whenever available, and \eqref{eq:constraint_marginal} implies that the data marginal distribution should be close to the marginal distribution $\mathbf{p}$. 
Then, for any marginal distribution $\mathbf{p}$ and $\epsilon\,{>}\,0$, it is easy to show $\mathcal{L}_{\mathrm{WDRO}}(\mathcal{C}_{\mathbf{p}, \epsilon})\,{\leq}\,\mathcal{L}_{\mathrm{WDRO}}(\mathcal{C})$ as $\mathcal{C}_{\mathbf{p}, \epsilon}\,{\subset}\,\mathcal{C}$.
Moreover, we will see in Lemma~\ref{lemma:proof_1} that, with high probability, the constraint set $\mathcal{C}_{\mathbf{p}^{\star}, \epsilon}$ with the true marginal distribution $\mathbf{p}^{\star}$ contains the true group assignment $\{g^{\star}\}$.
\begin{lemma}
\label{lemma:proof_1}
The constraint set $\mathcal{C}_{\mathbf{p}^{\star}, \epsilon}$ contains the true group labels $\{g_{i}^{\star}\}_{i=1}^{N}$ with high probability:
\begin{equation}
    P(\{g_{i}^{\star}\}_{i=1}^{N}\,{\in}\,\mathcal{C}_{\mathbf{p}^{\star}, \epsilon}) \geq 1\,{-}\,2e^{-2N\epsilon^{2}}\label{eq:proof_1}
\end{equation}
\end{lemma}
The proof is in Appendix~\ref{sec:proof_of_lemma}. As in \eqref{eq:proof_1}, the probability of the constraint set containing the true group labels gets closer to 1 by allowing a larger variance ($\epsilon$) from the true marginal distribution. For fixed $\epsilon\,{>}\,0$, the probability gets closer to 1 as we increase the number of \emph{unlabeled} data ($N$).

Finally, this implies that $\mathcal{L}_{\mathrm{WDRO}}(\mathcal{C}_{\mathbf{p}^{\star}, \epsilon})$ is an upper bound to that of {\groupdro}:
\begin{center}
\begin{tcolorbox}[minipage,colback=gray!20,arc=0pt,outer arc=0pt,text width=3in]
\centering
{
$\mathcal{L}_{\mathrm{GDRO}} \operatorname*{\leq}\limits_{\mathrm{w.h.p}} \mathcal{L}_{\mathrm{WDRO}}(\mathcal{C}_{\mathbf{p}^{\star}, \epsilon}) \leq \mathcal{L}_{\mathrm{WDRO}}(\mathcal{C})$
}
\end{tcolorbox}
\end{center}
In practice, however, the true marginal distribution $\mathbf{p}^{\star}$ may not be available. For our setting where group labels are partially available, with an assumption that group labels are missing completely at random (MCAR) \cite{rubin1976inference}, the true marginal distribution could be estimated from the subset of data with group labels. This again allows us to formulate a constraint set that contains the ground-truth group assignment with high probability.

To be more specific, let $\bar{\mathbf{p}}$ be the estimate of the marginal distribution from $\{(x_i,y_i,g_i^{\star})\}_{i=1}^{K}$. 

\begin{lemma}
\label{lemma:proof_2}
The constraint set $\mathcal{C}_{\bar{\mathbf{p}}, \delta+\epsilon}$ contains the true group labels $\{g^{\star}_{i}\}_{i=1}^{N}$ with high probability:
\begin{equation}
    P(\{g_{i}^{\star}\}_{i=1}^{N}\,{\in}\,\mathcal{C}_{\bar{\mathbf{p}}, \delta+\epsilon}) \geq 1\,{-}\,2e^{-2N\epsilon^{2}}\,{-}\,2e^{-2K\delta^{2}}\label{eq:proof_2}
\end{equation}
\end{lemma}
We provide a proof in Appendix~\ref{sec:proof_of_lemma}. Here, $\delta$ is introduced to take into account the estimation error of the true marginal distribution $\mathbf{p}^{\star}$. When $K$, the number of labeled data, is large, the bound in \eqref{eq:proof_2} is close to 1. 

\begin{center}
\begin{tcolorbox}[minipage,colback=gray!20,arc=0pt,outer arc=0pt,text width=4in]
\centering
{
Under MCAR, for large $K$:
$\mathcal{L}_{\mathrm{GDRO}} \operatorname*{\leq}\limits_{\mathrm{w.h.p}} \mathcal{L}_{\mathrm{WDRO}}(\mathcal{C}_{\bar{\mathbf{p}}, \delta+\epsilon})$
}
\end{tcolorbox}
\end{center}

\begin{algorithm}[t]
	\caption{{\worstoffdro} Algorithm}
	\label{alg:worstoffgropudro_alg}
	\begin{algorithmic}[1]
	    \STATE {\em Input:} Fully-labelled dataset $\{(x_{i}, y_{i}, g_{i}^{\star})\}_{i=1}^{K}$, task-labeled dataset $\{(x_{i}, y_{i}, -)\}_{i=K+1}^{N}$
		\STATE {\em Initialization:} learning rates $\eta_w$ and $\eta_q$, Marginal distribution $\bar{\mathbf{p}}$, $\epsilon$
		\STATE {\em Parameters:} Group Weights $q_j$, {\worstoffdro}  group assignments $\hat{g}$,\\ Neural network parameter $w$
		\FOR{$t = 0, 1, 2, ...,T$}
		\STATE $\{\hat{g}^{t}\} \leftarrow \max_{\{\hat{g}\}\in\mathcal{C}_{\bar{\mathbf{p}},\epsilon}} \sum_{j=1}^{M} q^{t}_j \frac{\sum_i \hat g_{ij} l(x_i, y_i; w^{t})}{\sum_i \hat g_{ij}}$ where, $\mathcal{C}_{\bar{\mathbf{p}},\epsilon}$ as defined in \eqref{eq:worstoffdro_constraint_set_with_marginal}.  
		\STATE  Gradient descent on $w$:\\
		        $w^{t+1} \leftarrow w^{t} - \eta_w \mathlarger{\nabla_w} \sum_{j=1}^{M} q^{t}_j \frac{\sum_i \hat g^{t}_{ij } l(x_i, y_i;w)}{\sum_i \hat g^{t}_{ij}}$
		\STATE Exponential ascent on $q$:\\
		       $q^{t+1} \leftarrow q^{t} \exp{(\eta_q \mathlarger{\nabla_q} \sum_{j=1}^{M} q_j \frac{\sum_i \hat g^{t}_{ij} l(x_i, y_i;w^{t+1})}{\sum_i \hat g^{t}_{ij}})} $
		\ENDFOR
		\STATE {\em Output:} Trained neural network parameters $w^T$
	\end{algorithmic}
\end{algorithm}

\vspace{-0.05in}
\subsection{A Practical Optimization Algorithm}
\label{sec:method_optimization}
\vspace{-0.05in}
We are interested in solving the optimization problem $\mathcal{L}_{\mathrm{WDRO}}(\mathcal{C}_{\mathbf{p},\epsilon})$. Unfortunately, the inner maximization problem with respect to the group assignments $\{\hat{g}\}$ in \eqref{eq:worstoffdro_nocon} is challenging as variables are discrete and the objective cannot be decomposed due to the marginal distribution constraint.
In this section, we describe an optimization recipe with a few approximations. 

First, we propose to use a soft group assignments. This not only converts the problem into continuous optimization problem, but also accommodates inherent uncertainties in group assignment for data with unlabelled group labels. Specifically, for each sample, we retain a soft group assignment $\hat{g}_{i}\,{\in}\,\Delta^{M}$, and optimize the {\worstoffdro} objective over the constraint set $\mathcal{C}_{\bar p, \epsilon}$ as defined below:
%
\begin{equation}
    \min_{w}\max_{q\in\Delta^{M}} \max_{\{\hat{g}\}\in\mathcal{C}_{\mathbf{p},\epsilon}}\sum_{j=1}^{M} q_{j} \Big[\frac{\sum_{i=1}^{N} \hat{g}_{ij} l(x_{i}, y_{i}, w)}{\sum_{i=1}^{N} \hat{g}_{ij}}\Big]\label{eq:worstoffdro_soft}
\end{equation}
where the constraint set is defined as:
\begin{equation}
\mathcal{C}_{\mathbf{\bar p},\epsilon}\,{=}\,\Bigg\{\{\hat{g}_{i}\}_{i=1}^{N}\Big\vert\begin{array}{l}
        \;\hat{g}_{i}\,{\in}\,\Delta^{M},\forall i\,{\leq}\,N, \\
        \;\hat{g}_{i(g_{i}^{\star})}\,{=}\,1, \forall i\,{\leq}\,K,\\
        \;\vert\frac{1}{N}\sum_{i=1}^{N}\hat{g}_{ij}\,{-}\,\mathbf{\bar p}_{j}\vert\,{\leq}\,\epsilon, \forall j\,{\leq}\,M
    \end{array}\Bigg\}\label{eq:worstoffdro_constraint_set_with_marginal}
\end{equation}
%

The first condition ensures assignments in the probability simplex, second one ensures assignments are consistent with ground-truth for labeled data, and the third one validates the data marginal distribution follows the provided distribution. The third constraint also provides for a mitigation strategy when $\mathbf{\bar p}$ is misspecified (likely when data is not MCAR). The $\epsilon$ in the third constraint is a hyper-parameter. Increasing $\epsilon$ provides more flexibility with the choice of assignments (more details in Appendix~\ref{sec:increase_set}).

We alternate optimization over $w$, $q$ and $\{\hat{g}\}$ as shown in Algorithm~\ref{alg:worstoffgropudro_alg}. That is, we first solve inner maximization over $\{\hat{g}\}$, and conduct gradient descent on $w$ and the exponential gradient ascent on $q$, and iterate. An exponential ascent on $q$ achieves smaller losses for linear predictors (like $q$) \citep{kivinen1997exponentiated}. The inner maximization over $\{\hat{g}\}$ is solved using off-the-shelf CVXPY solver~\citep{diamond2016cvxpy} (more details in Appendix~\ref{sec:opti_notes}). 

Next, let us see how the worst-off assignments computed by the algorithm look to be. For simplicity, consider the case $\epsilon=0$ and $K=0$, (i.e., no labelled groups). Denoting $\frac{q_j}{\sum_i^N \hat g_{ij}} = \frac{N q_j}{\mathbf{\bar p}_{j}}$ in \eqref{eq:worstoffdro_soft} as $\theta_j$ and $l(x_i, y_i, w)$ as $l_i$, we can re-write the maximization over $\{\hat{g}_{i j}\}$ as,
\begin{equation}
 \max_{\{\hat{g}_{ij}\}\in\mathcal{C}_{\mathbf{p},\epsilon=0}}\sum_{i=1, j=1}^{N, M} \hat{g}_{ij} \times \theta_j \times l_i \label{eq:worstoffdro_rewrite}
\end{equation}
The constraints ensure that $\sum_{i=1}^{N}\hat{g}_{ij}=N\mathbf{\bar p}_{j}$ and   $\sum_{i=j}^{M}\hat{g}_{ij} = 1, \hat{g}_{ij}\ge 0$ for all $i\le N$ and $j \le M$ respectively. The linear program \eqref{eq:worstoffdro_rewrite} sets the highest mass on $\hat{g}_{ij}$ for $i$ and $j$ that maximize $\theta_j$ and sample loss $l_i$. A large $\theta_j$ represents groups with a high group weight $q_j$ and low marginal probability $\mathbf{\bar p}_{j}$, characteristic of a worst-off group. In summary, we find that \textit{high loss samples are assigned to groups with high group weights and low marginal probabilities}. We provide a detailed example of this observation in the Appendix~\ref{sec:worstoffdro_example} and discuss the case where marginal constraints are ignored. 
%

\section{Experiments}
\label{sec:exp}
%
We test the efficacy of our method on image and tabular datasets, each of which consists of samples from mutually exclusive groups or environments. These groups are indicative of the background or an RGB identification for image datasets, and attributes such as gender or race for tabular datasets. As discussed in Section~\ref{sec:intro}, one or more of the available groups form a minority in terms of the sample size and demographics. The presence of minority groups results in a possible scenario where the aggregate performance is (falsely) remarkable, because evaluations are dominated by larger groups, even though the performance on the minority groups is poor. In our experiments we assume group numbers to be known but group labels are \textit{missing completely at random} at a fixed rate at train time.  

In Section~\ref{sec:exp_setting}, we outline our baselines for comparison, our metrics of evaluation and the model selection strategy. In Section~\ref{sec:exp_results}, we describe each dataset in detail and highlight the differences across the groups within the dataset. All the quantitative results are available in Table~\ref{tab:quant_results} and per-group summary statistics are present in Table~\ref{tab:samples}. More analysis of our method is provided in Section~\ref{sec:exp_ablation}.

\begin{table}[!t]
	\centering
	\resizebox{0.95\columnwidth}{!}{%
		\begin{tabular}{c c c c c c c}
		\toprule
			Dataset   & \# Labeled            & \# UnLabeled            &  Total samples                      & \# Groups & \# Minority Samples & \# Majority Samples \\ \midrule\midrule
			Waterbirds      & $508$ & $4287$ & $4795$ & $4$ & $5$ & $113$\\
			CMNIST  & $3983$ & $35657$ & $39640$         & $3$ & $276$ & $2138$ \\
			Adult & $1308$ & $11635$ & $12943$      & $4$ & $63$& $823$ \\
			CelebA & $8030$ & $154740$ & $162770$  & $4$ & $62$ & $3547$ \\ 
		\bottomrule
		\end{tabular}%
	}
	\caption{\label{tab:samples} \small \textbf{Dataset description.} We show sample counts in labelled and unlabelled training sets, as well as counts for majority and minority groups. The number of labelled samples are about $10\%$ of total samples.}
	\vspace{-0.25in}
\end{table}

\vspace{-.1in}
\subsection{Experimental Settings}
\label{sec:exp_setting}
\vspace{-.05in}

{\bf Baselines.}
 We contrast the performance of our method with respect to a few well-known baselines. \begin{enumerate}[topsep=0pt,itemsep=-1ex,partopsep=1ex,parsep=1ex,align=left]
 \item \textbf{ERM:} Empirical Risk Minimization that optimizes aggregate average loss over all the samples in the training dataset. 
 \item \textbf{{\unsupdro} \citep{hashimoto2018fairness}:} Samples with losses exceeding a threshold $\eta$ are considered as a group whose average loss is optimized. Since the method doesn't require group labels, it is an unsupervised algorithm. The method {\unsupdro}, similar to CVaR DRO \cite{levy2020large}, requires a wider hyper-parameter search relative to the other baselines. More details in Appendix~\ref{sec:unsupdro_details}.
 \item \textbf{{\groupdro} \citep{sagawa2019distributionally}:} A method that optimizes the Rawlsian criterion by assigning simplex weights to the groups. 
 The group labels for individual samples are assumed to be available, hence this method is fully-supervised in terms of the group label. 
 \item \textbf{{\groupdro} (Partial):} We consider another variant of {\groupdro} that only uses samples with group labels at train time. We call the method \textit{{\groupdro} (Partial)}, to contrast with the above baseline, \textit{{\groupdro} (Oracle)}.
 \end{enumerate}
 
 We compare above methods with our proposal, \textbf{{\worstoffdro}}. Note that our approach requires marginal probabilities as an input to the algorithm, which are computed from the training dataset in our experiments. All our baselines for experiments are single-stage approaches similar to our method {\worstoffdro}. We provide a comparison to two-stage methods in Section~\ref{sec:twostage}.  For baselines and {\worstoffdro} implementations, samples are drawn randomly for every batch ensuring an unbiased comparison to the ERM baseline. This is unlike \citep{sagawa2019distributionally} who adopt a weighted sampling procedure which could be noisy when group labels are uncertain or missing as in our problem. \footnote{Minor differences, in the accuracies of the {\groupdro} baseline to those reported in \citep{sagawa2019distributionally}, are due to the random sampling scheme of samples during batch-wise updates. }
 
{\bf Evaluation Metrics.}
We set aside a test set whose group labels are fully available. Since all of our datasets characterize a classification task, we evaluate overall accuracies and per-group accuracies in our experiments. Specifically, we highlight the accuracy of the minority group (\textbf{min}) together with the overall (\textbf{avg}) accuracy where individual samples are equally weighted regardless of their group.

{\bf Model Selection.}
Model selection plays a crucial role when distributional differences are observed in a dataset \citep{gulrajani2021in}. In our problem setting, individual groups may differ from each other in the joint distribution over the data and the label space, however, the testing set resembles the training set. That is, there are no out of distribution samples and the focus is to improve robustness over a predefined set of groups common to both training and testing datasets. Consequently, among the recommendations made in \citep{gulrajani2021in}, a training domain validation set is a feasible strategy for our problem. In our algorithm, learning rate and weight decay are important hyper-parameters. Prior works \citep{sagawa2019distributionally,sagawa2020investigation} noted that ERM fails to optimize for the minority group's performance under high regularization regime thus necessitating an alternative. Hence, we compare our methods in this regime. We also tune for algorithmic specific hyper-parameters for each baseline. These hyper-parameters are the loss threshold in {\unsupdro} ($\eta_\text{UDRO}$) and the step size for the group weights in {\groupdro} ($\eta_\text{GDRO}$) and in {\worstoffdro} ($\eta_\text{WDRO}$). A list of all hyper-parameter choices used in the experiments is provided in the Appendix~\ref{sec:app_hparam_tuning}. All the numbers reported in the paper were averaged over three random seeds. 

We adopt \textbf{NVP} (novel validation procedure) \citep{NEURIPS2018_donini} in our experiments. In this procedure, we first search for hyper-parameters with the best overall accuracy. Then, from the top five best performing hyper-parameters, we select the model that achieves the highest minority group accuracy. Such a procedure offers robustness to hyper-parameters in the reported numbers. 

\vspace{-0.05in}
\subsection{Quantitative Results}
\label{sec:exp_results}
\vspace{-0.05in}

We describe key results on four datasets\footnote{CMNIST and Adult datasets differ from their previous instantiations in \cite{creager2021environment}. These datasets are used to assess group robustness (see Sec~\ref{sec:related}), hence same set of pre-defined groups are used in the training and testing phases. }, Waterbirds~\citep{sagawa2019distributionally}, Group CMNIST~\citep{arjovsky2019invariant}, Group Adult~\citep{dua2017uci}, and CelebA~\citep{liu2015deep}.

\vspace{-0.05in}
\subsubsection{Waterbirds Dataset}
\label{sec:exp_results_waterbirds}
\vspace{-0.05in}

The dateset, used in \citep{sagawa2019distributionally}, comprises of $4795$ images of birds from the CUB dataset \citep{welinder2010caltech} and the backgrounds taken from the Places dataset \citep{zhou2017places}. Each image in the dataset has a background of land or water. The target labels are either ``landbirds'' or ``waterbirds''. The authors in \citep{sagawa2019distributionally} create four groups with each target label and a background class considered as a group. In this dataset the groups ``landbirds'' on water and ``waterbirds'' on land form a minority. Our results in Table~\ref{tab:quant_results} firstly shows that the ERM method attains a small minority group accuracy of $60\%$. All the invariant learning baselines, except for {\groupdro} (Partial), improve the minority group's accuracy. Next, we observe that in comparison to {\groupdro} (Partial), our proposed {\worstoffdro} improves the minority group's performance by a significant margin of $21\%$. Due to this improvement, the all-group accuracy also improves by $8\%$. Minority group's performance on fully-supervised method {\groupdro} (Oracle) is at $83\%$ accuracy with a window of $18\%$ difference from {\worstoffdro}. Lastly, Appendix~\ref{sec:increase_set} describes experiments where the constraint set size is gradually increased by varying the $\epsilon$ parameter of $\mathcal{C}_{\mathbf{\bar p},\epsilon}$. Increasing $\epsilon$ parameter accommodates for the setting where $\mathbf{\bar p}$ is misspecified.

\begin{table*}[t] 
    \centering
	\footnotesize
	\begin{tabular*}{0.95\textwidth}{l @{\extracolsep{\fill}}
			*{16}{S[table-format=3.0]}}
		\toprule
		& \multicolumn{2}{c}{Waterbirds} & \multicolumn{2}{c}{CMNIST} & \multicolumn{2}{c}{Adult} & \multicolumn{2}{c}{CelebA}\\
		\cmidrule{2-3} \cmidrule{4-5} \cmidrule{6-7} \cmidrule{8-9} 
		& {\cellcolor{gray!25}\textbf{min}}  & {\textbf{avg}} & {\cellcolor{gray!25}\textbf{min}} & {\textbf{avg}} & {\cellcolor{gray!25}\textbf{min}} &  {\textbf{avg}} & {\cellcolor{gray!25}\textbf{min}} &  {\textbf{avg}}\\ 
		\midrule\midrule
		{\groupdro} (Oracle)   & \cellcolor{gray!25}{$83$} &  {$92$}  &  {\cellcolor{gray!25}$50$}  &  {$75$}  &  {\cellcolor{gray!25}$82$}  &  {$88$}  & {\cellcolor{gray!25}$80$}  &  {$94$}  \\ \midrule
		ERM                    & {\cellcolor{gray!25}$60$} &  {$87$}  &  {\cellcolor{gray!25}$13$}  &  {$79$}  &  {\cellcolor{gray!25}$68$}  &  {$92$}  & {\cellcolor{gray!25}$45$}  &  {$95$}  \\
		{\unsupdro}            & {\cellcolor{gray!25}$65$} &  {$88$}  &  {\cellcolor{gray!25}$10$}  &  {$80$}  &  {\cellcolor{gray!25}$68$}  &  {$92$}  & {\cellcolor{gray!25}$39$}  &  {$96$}  \\
		{\groupdro} (Partial)  & {\cellcolor{gray!25}$44$} &  {$81$}  &  {\cellcolor{gray!25}$36$}  &  {$76$}  &  {\cellcolor{gray!25}$67$}  &  {$90$}  &   {\cellcolor{gray!25}$40$}  &  {$95$} \\ 
		{\worstoffdro}         & {\cellcolor{gray!25}$\bm{65}$} &  {${89}$}  &  {\cellcolor{gray!25}{$\bm{39}$}}  &  {$77$}  &  {\cellcolor{gray!25}$\bm{71}$}  &  {$91$}  &   {\cellcolor{gray!25}$\bm{49}$}  &  {$95$}    \\ 	
		\bottomrule
	\end{tabular*} 
	\vspace{-0.05in}
	\caption{ \label{tab:quant_results} \small \textbf{Quantitative Results.} For baselines, we consider an ERM, {\unsupdro} \citep{hashimoto2018fairness}, {\groupdro} (Partial) for partly labelled {\groupdro} \citep{sagawa2019distributionally} method, {\groupdro} (Oracle) for the fully supervised model. Our method {\worstoffdro} improves the minority group's accuracy (\textbf{min}) while maintaining a similar overall accuracy (\textbf{avg}) relative to baselines. The accuracies are computed on the test set and are an average over three random runs. The standard deviations are provided in the Appendix Table~\ref{tab:quant_results_std}.} 
	\vspace{-0.2in}
\end{table*} 

\vspace{-0.05in}
\subsubsection{Group CMNIST Dataset}
\label{sec:exp_results_cmnist}
\vspace{-0.05in}

CMNIST, derived from an MNIST~\citep{lecun1998gradient}, is a digit recognition dataset where each image is colored either red or green. Digits $<5 / \ge 5$ are considered as label $0 / 1$. We consider three groups in our experiments. In the first two groups, label $0$ images are predominantly colored red and vice versa. In the third group, which forms a minority, we switch coloring such that the label $1$ images are predominantly colored red. Specifically, for the first two groups, the color id is sampled by flipping the target label with probabilities $0.2$ and $0.1$ respectively, while the third group with probability $0.9$. Both training and testing sets contain three groups. The overall setup for generating a given group is similar to \citep{arjovsky2019invariant}. We show the results on CMNIST in Table~\ref{tab:quant_results}. Similar to the Waterbirds dataset, {\worstoffdro} improves the minority group's accuracy compared to the ERM method. Relative to {\groupdro} (Partial), {\worstoffdro} improves the accuracy of the minority group by $4\%$ and $1\%$ in the overall accuracy. The margin between {\worstoffdro} and {\groupdro} (Oracle) is $11\%$. Among all the baseline, {\unsupdro} attains lowest minority group accuracy of $10\%$. A large trade-off between the minority group accuracy and the all-group accuracy was seen for {\unsupdro} in this dataset. 

\vspace{-0.05in}
\subsubsection{Group Adult Dataset}
\label{sec:exp_results_adult}
\vspace{-0.05in}

We use a semi-synthetic version of the Adult dataset \citep{dua2017uci} for this experiment. Similar to \citep{NEURIPS2020_07fc15c9}, we consider race and sex as the four demographic groups. The target label is income $> 50K\$$ and is treated as label $1$. Similar to the CMNIST dataset, each group has a different correlation strength to the target label. For the purposes of the experiment, we exaggerate these spurious correlations caused by group membership close to \citep{creager2021environment}. Particularly, for samples with group label as Afican-American, we undersample examples with probability $P(y=1\mid group)=0.06$ whereas for the non African-American group labels, we oversample examples with probability $P(y=1\mid group)=0.94$. Table~\ref{tab:quant_results} indicates a $5\%$ improvement in the minority group's accuracy while maintaining the similar overall accuracy of $90\%$ compared to {\groupdro} (Partial). The {\groupdro} (Oracle) method reaches an accuracy of $82\%$ for the minority group compared to {\worstoffdro} which achieves $71\%$. Evidently, ERM underperforms in terms of the minority group's accuracy and attains about $68\%$ accuracy. 

\vspace{-0.05in}
\subsubsection{CelebA Dataset}
\label{sec:exp_results_celeba}
\vspace{-0.05in}

CelebA \citep{liu2015deep} is a dataset containing about 200k celebrity faces curated from the internet. There are $40$ labels available in this dataset which are annotated by a group of paid adult participants \citep{bohlen2017server}. Similar to \citep{sagawa2019distributionally}, we aim to predict the target attribute \textit{Blond Hair} that is spuriously correlated to the Gender attribute. Specifically, having blond hair correlates with the female attribute. The minority group in this dataset are the images with attributes (blond, male). The proportion of samples in the minority and the majority group is show in Table~\ref{tab:samples}. The quantiative results in Table~\ref{tab:quant_results} indicate an improvement of $9\%$ over the {\groupdro} (Partial) method for the proposed {\worstoffdro} algorithm. The {\groupdro} (Oracle) method achieves the highest minority group accuracy of $80\%$. The minority group performance for the ERM method, with an accuracy of $45\%$, is comparable to {\groupdro} (Partial). 
All the methods are similar in terms of the average group accuracy with values $>90\%$.

\vspace{-0.05in}
\subsection{Ablation Studies}
\label{sec:exp_ablation}
\vspace{-0.05in}
In this section, we discuss different components of our algorithm that influences it's performance.

\vspace{-0.05in}
\subsubsection{Increasing the labelled samples.}
\label{sec:exp_ablation_labeled_samples}
\vspace{-0.05in}
Recall that for the quantitative results in Table~\ref{tab:quant_results}, the number of labelled samples were around $10\%$ of the total training samples. In this section, we investigate the effects of increasing the number of labelled samples provided to the training algorithm. Although obtaining annotations for groups is an arduous task \citep{NEURIPS2020_07fc15c9}, having more labelled groups provides two benefits for the algorithm. Firstly, the standard deviation of errors in estimating the marginal probabilities from the labelled portion of the data reduces \citep{wasserman2004all} ( $\approx \sqrt{\text{\# samples}}$ rate). Secondly, labelled groups reinforce an accurate evaluation of the Rawlsian objective in \eqref{eq:groupdro_reform} and appropriate weight updates for the groups. The results shown in Figure~\ref{fig:labpercent_minority} depict the minority group accuracy at different labelled percent thresholds. The corresponding plots for average group accuracies are provided in the Appendix Figure ~\ref{fig:labpercent_average}. {\worstoffdro} method is compared with {\groupdro} (Partial). It is observed that for both the methods, the minority group's accuracy increases with more labelled data. Furthermore, the accuracy values for {\worstoffdro} method are better than {\groupdro} at several thresholds. The methods converge at a threshold specific to the datasets. Increasing the labelled counts beyond such a threshold saturates the {\worstoffdro} performance, however, {\groupdro} (Partial) consistently improves until Oracle performance is attained.

\vspace{-0.05in}
\subsubsection{Minority Group vs. Overall Accuracy.}
\label{sec:exp_ablation_minor_vs_overall}
\vspace{-0.05in}

As discussed in Section~\ref{sec:intro}, several groups in the training dataset, especially the minority groups, could be distributionally different from the majority group samples. Consequently, a mild tradeoff surfaces between the minority group accuracy values and the aggregate accuracies. Addressing this issue, recall that we leverage a robust model selection criterion such as NVP (see Section~\ref{sec:exp}, model selection paragraph) that balances both the minority and aggregate group accuracies. We extend the results in this section, by plotting evaluations at different hyper-parameter choices for our algorithm. Figure~\ref{fig:pareto} contrasts minority group and overall accuracy across all the datasets. Evidently, the top-right corners are desirable regions for the models to be present with maximum performance across both the metrics. The Adult dataset in Figure~\ref{fig:pareto} shows a clear envelope on the {\worstoffdro} models that surpass the corresponding {\groupdro} models. Similar trend exists on the remaining datasets with more {\worstoffdro} models concentrated in the top-right corner.

\begin{figure*}[t]
    \centering
	\begin{subfigure}[b]{0.24\linewidth}
	\includegraphics[width=\linewidth]{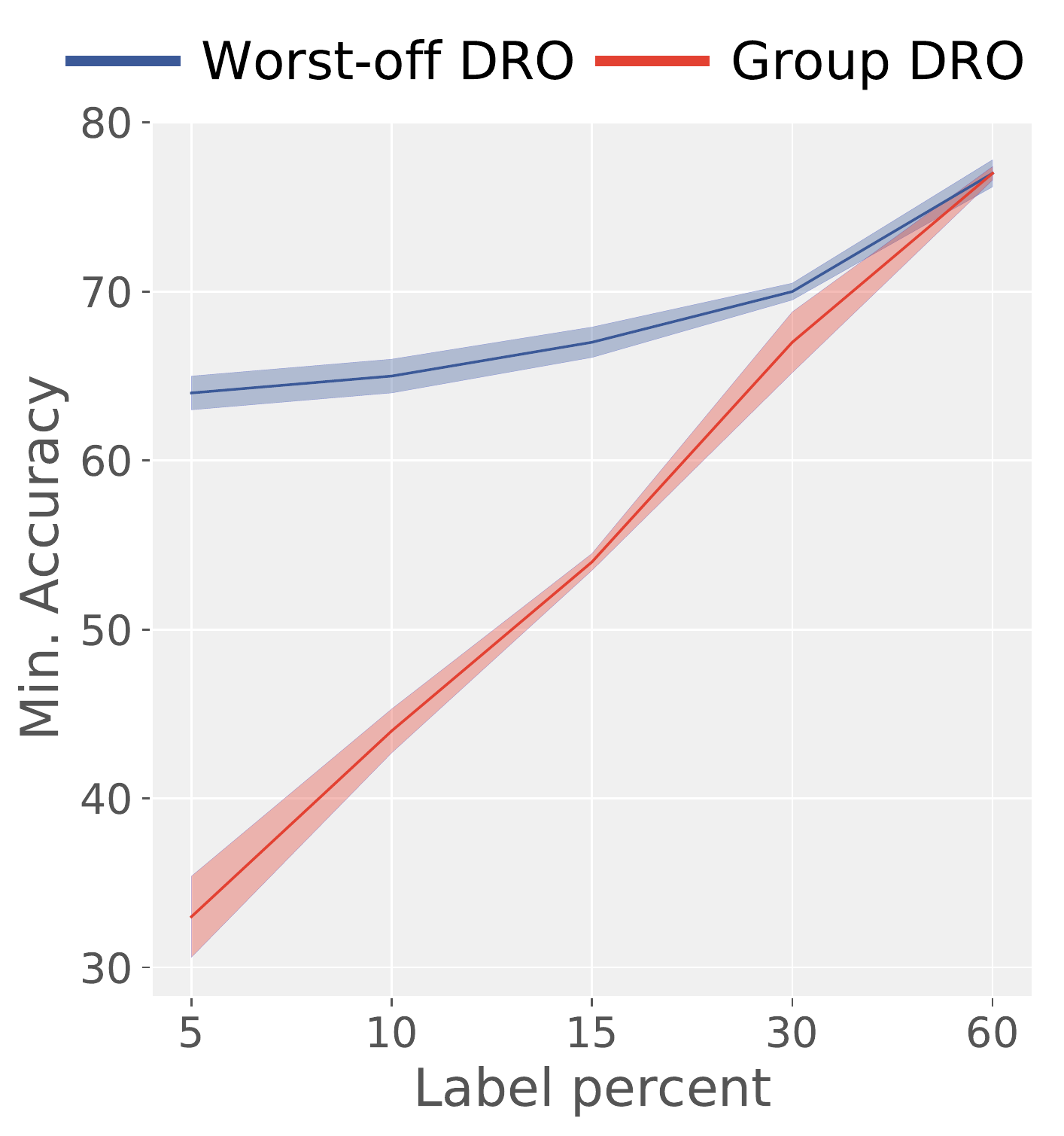}
	\caption{Waterbirds}
	\end{subfigure} 
	\begin{subfigure}[b]{0.24\linewidth}
	\includegraphics[width=\linewidth]{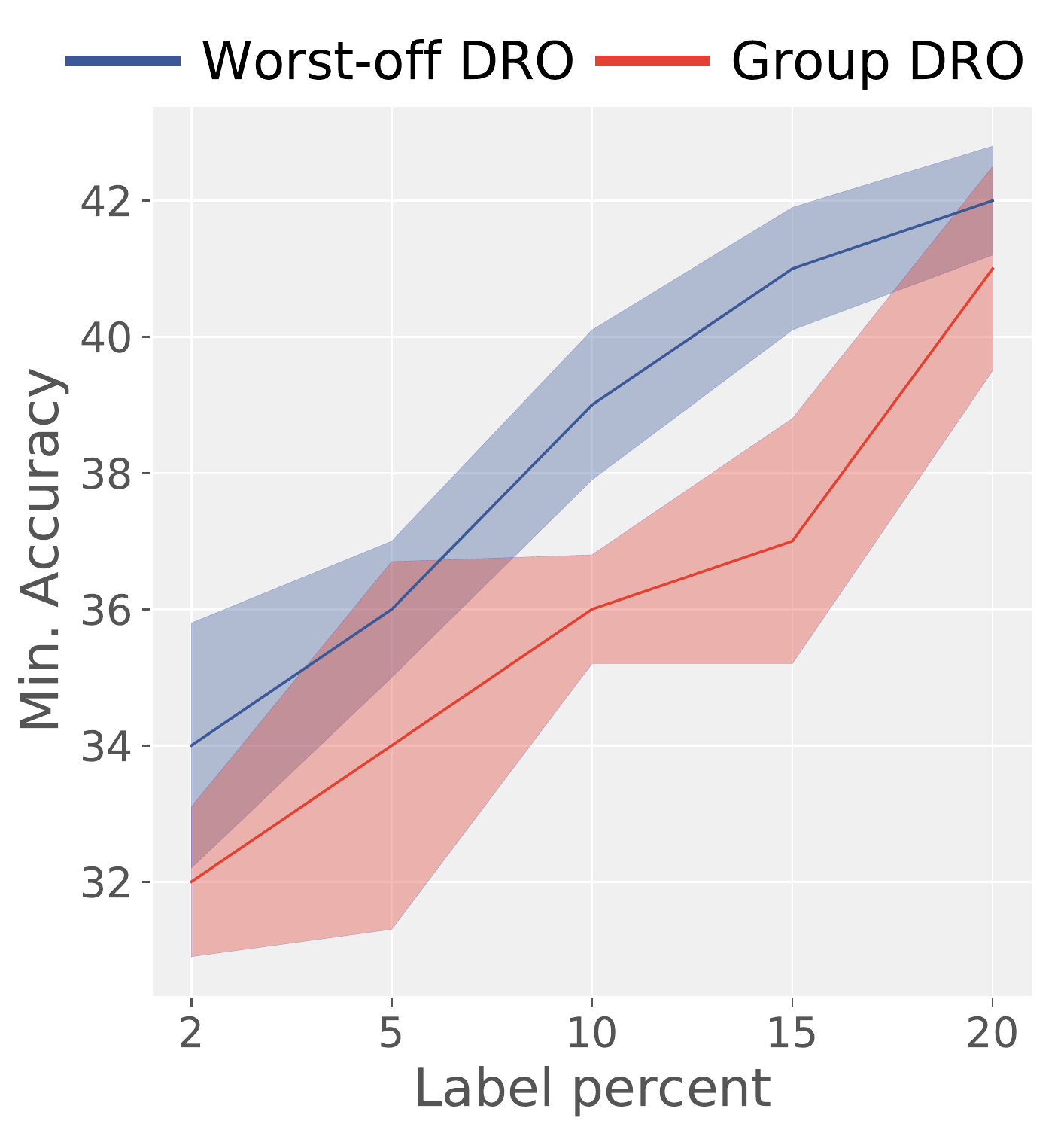}
	\caption{CMNIST}
	\end{subfigure}
	\begin{subfigure}[b]{0.24\linewidth}
	\includegraphics[width=\linewidth]{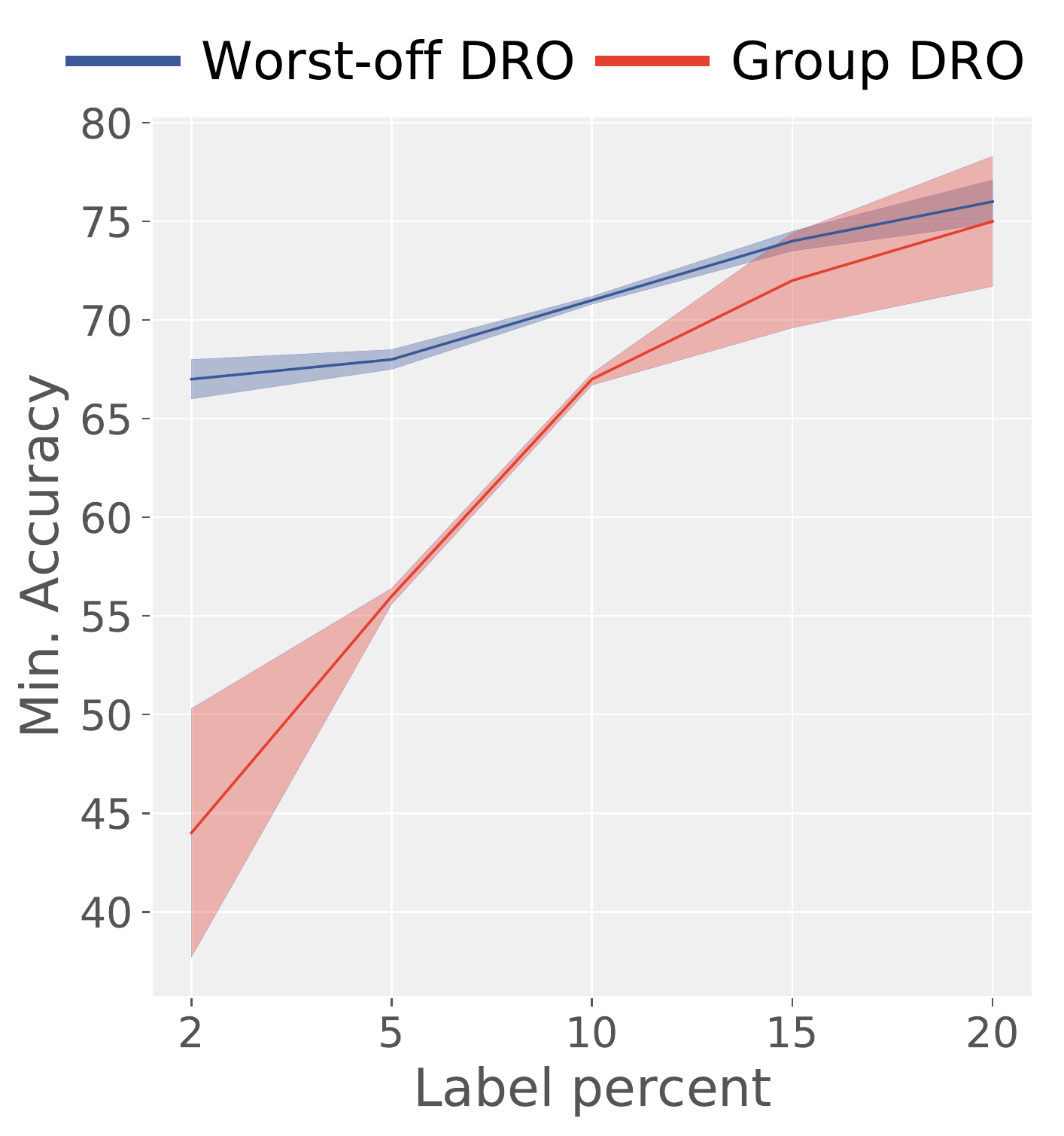}
	\caption{Adult}
	\end{subfigure}
	\begin{subfigure}[b]{0.24\linewidth}
	\includegraphics[width=\linewidth]{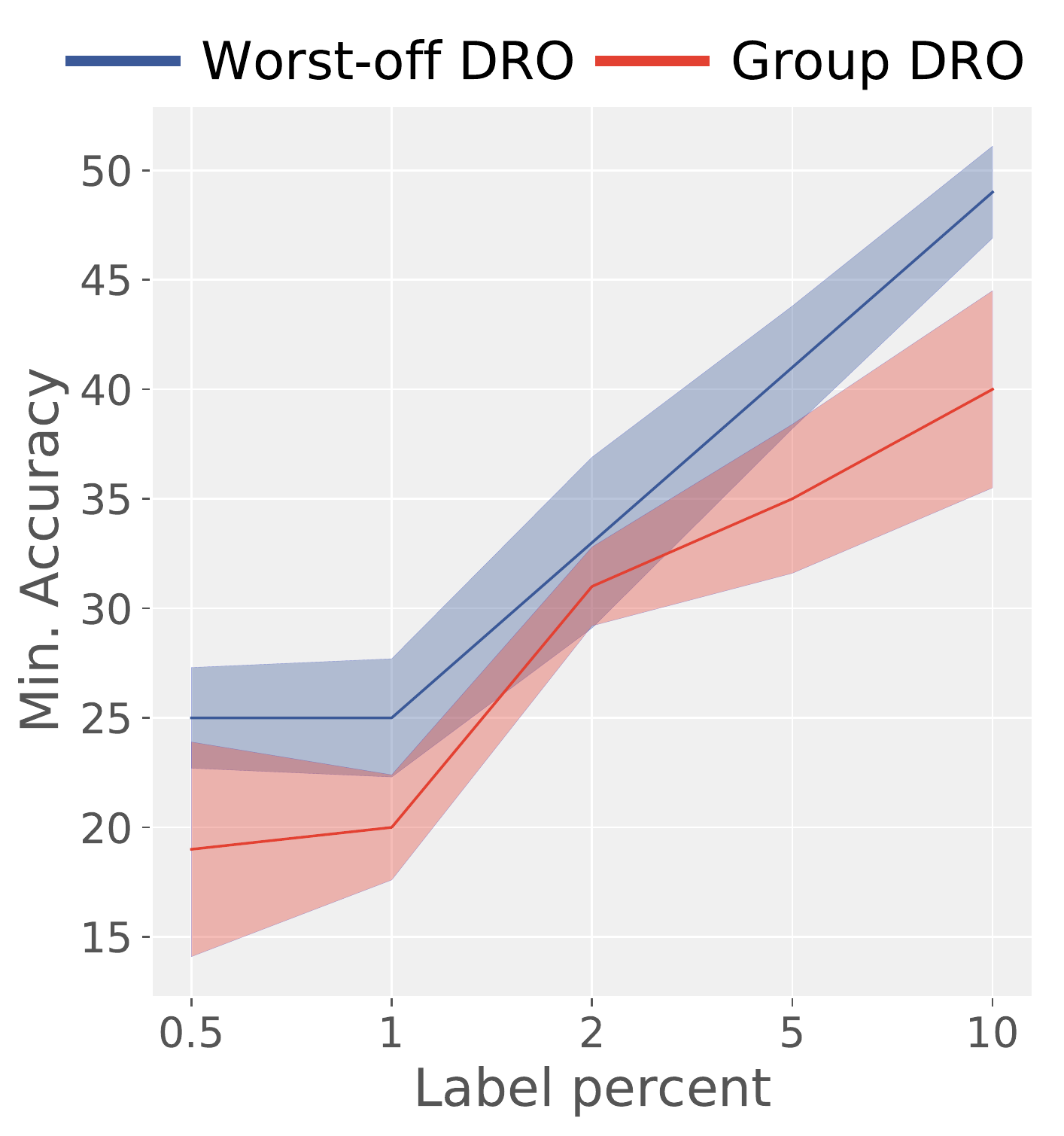}
	\caption{CelebA}
	\end{subfigure}%
	\vspace{-0.05in}
	\caption{ \textbf{Increasing the labelled samples.}  Minority group accuracies are plotted at different counts of the labelled samples in the training dataset. Both, {\groupdro} (Partial) and {\worstoffdro} algorithms improve the minority group accuracies with more training labels. Also, the {\worstoffdro} method has higher accuracy values than {\groupdro} method. The aggregate group accuracies are shown in the Appendix Figure ~\ref{fig:labpercent_average}. }
	\label{fig:labpercent_minority}
	\vspace{-0.2in}
\end{figure*}
\begin{figure*}[!t]
    \centering
	\begin{subfigure}[b]{0.24\linewidth}
	\includegraphics[width=\linewidth]{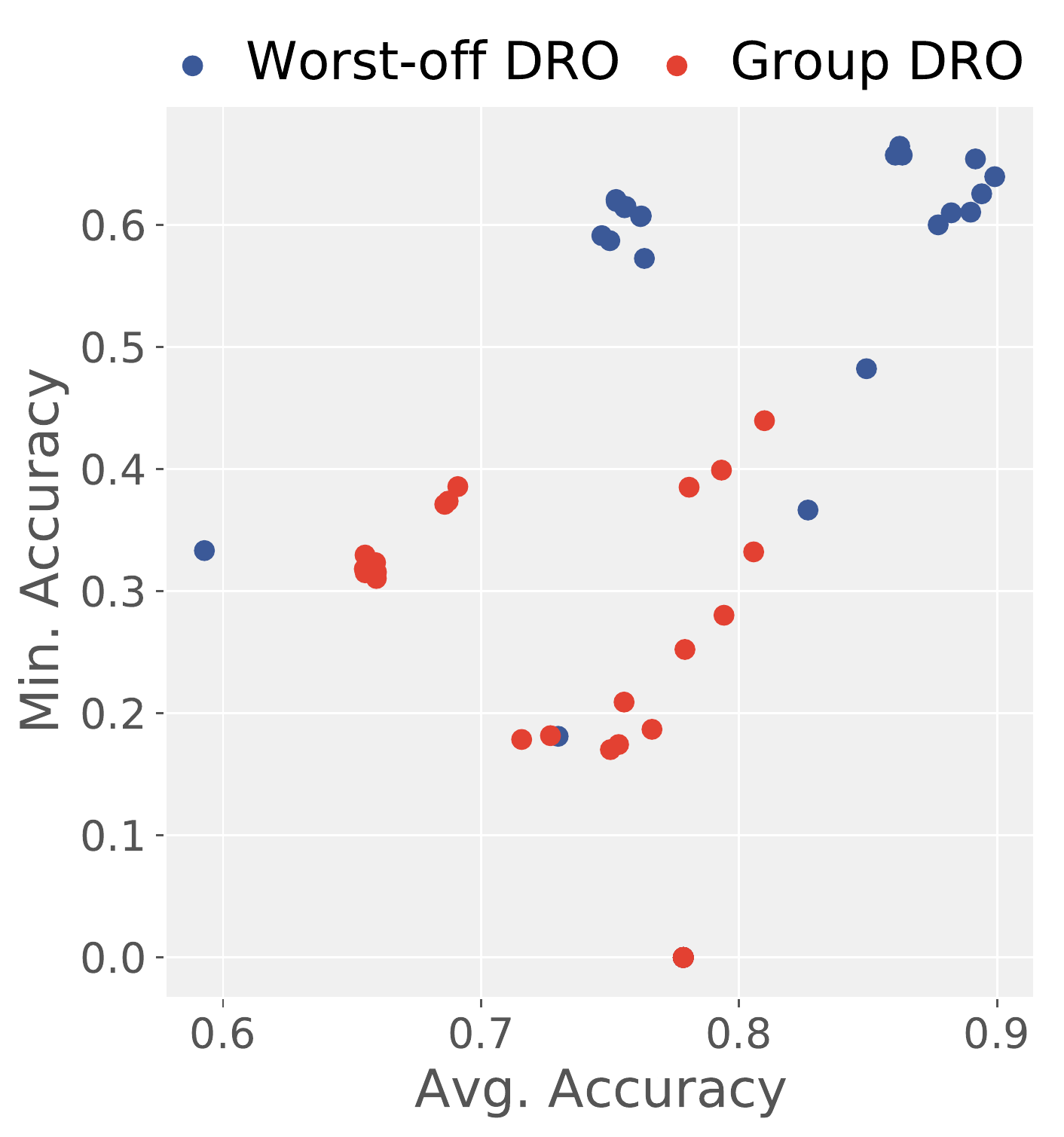}
	\caption{Waterbirds}
	\end{subfigure} 
	\begin{subfigure}[b]{0.24\linewidth}
	\includegraphics[width=\linewidth]{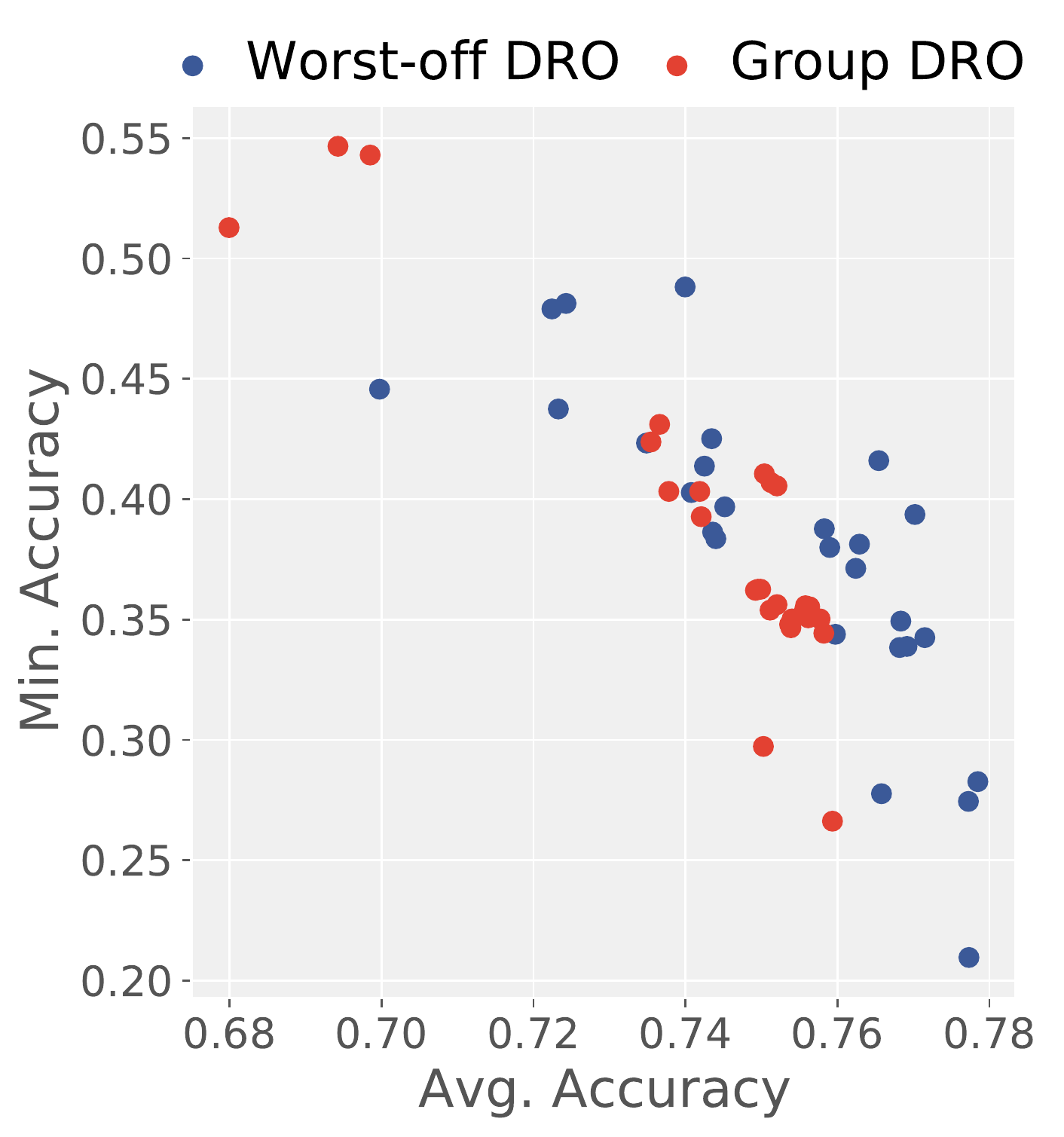}
	\caption{CMNIST}
	\end{subfigure}
	\begin{subfigure}[b]{0.24\linewidth}
	\includegraphics[width=\linewidth]{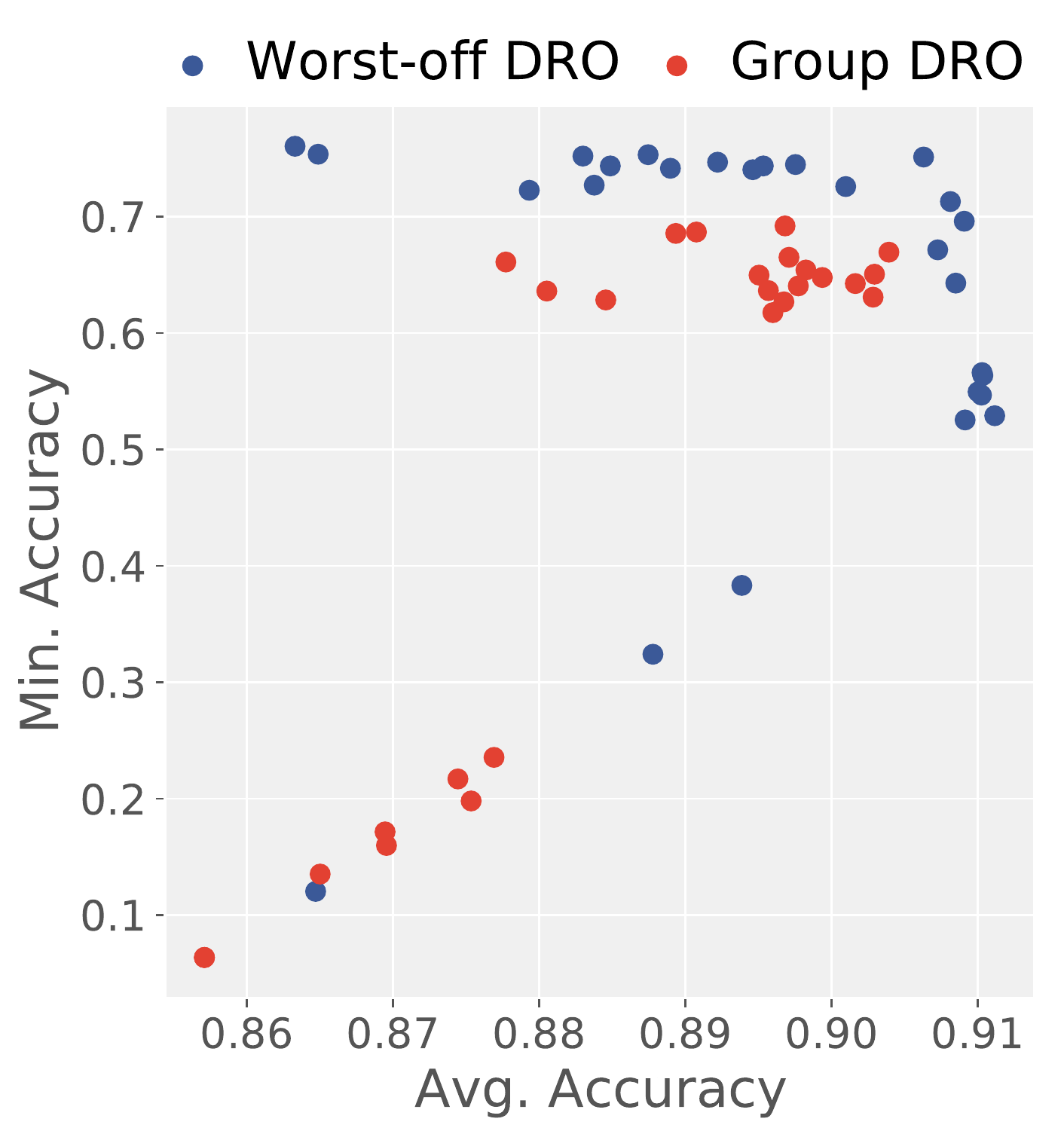}
	\caption{Adult}
	\end{subfigure}
	\begin{subfigure}[b]{0.24\linewidth}
	\includegraphics[width=\linewidth]{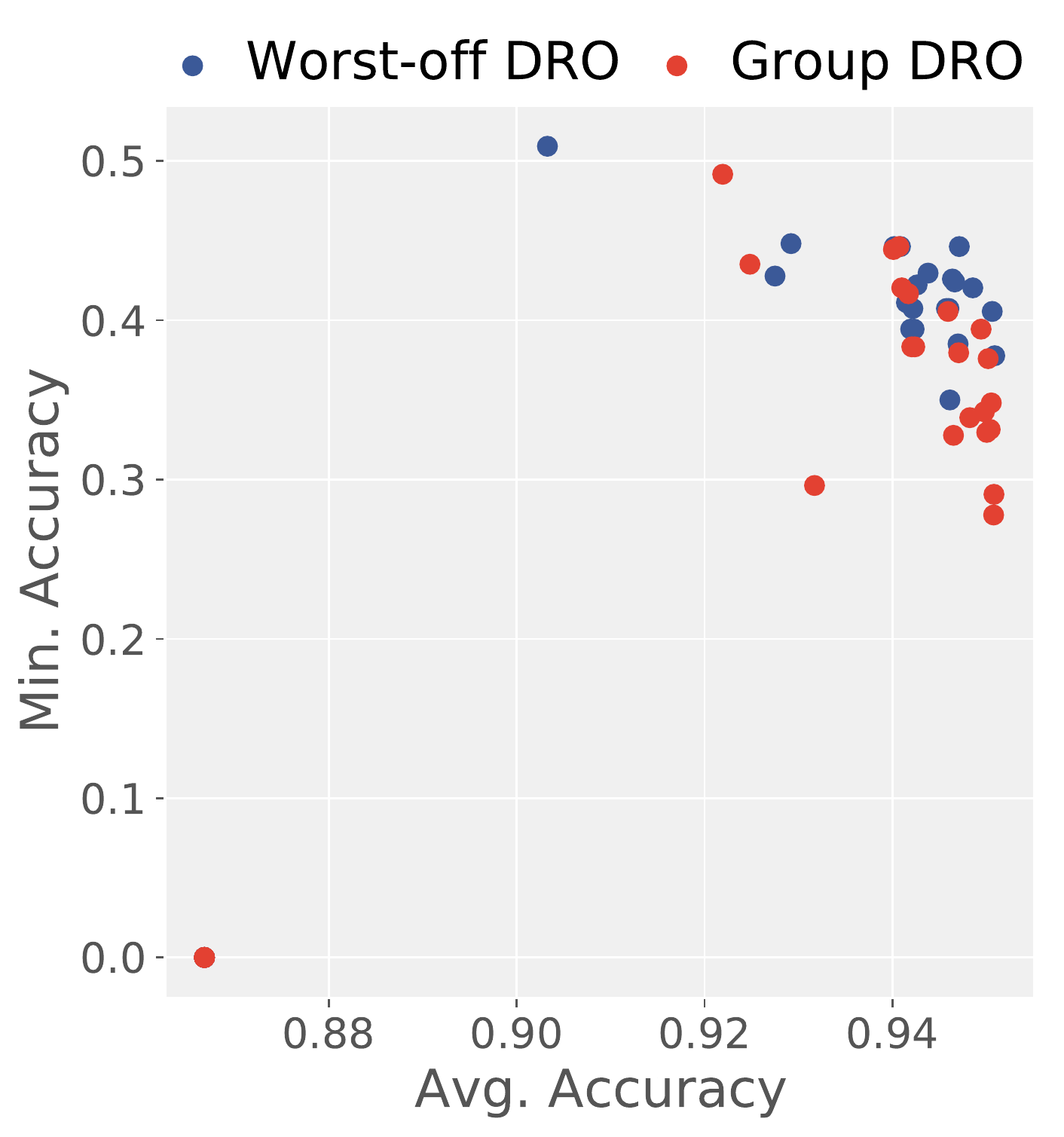}
	\caption{CelebA}
	\end{subfigure}%
	\vspace{-0.05in}
	\caption{ \textbf{Minority Group vs. Average Accuracy.}  Evaluations for different hyper-parameter choices are plotted for {\worstoffdro} and {\groupdro} (Partial) methods. Models from {\worstoffdro} training are concentrated in the top-right corner of the plots. This is desirable indicating a high accuracies across the two metrics. For model selection from among the possible choices, we adopt the NVP procedure (see Section~\ref{sec:exp}). }
	\label{fig:pareto}
	\vspace{-0.1in}
\end{figure*}
\begin{figure*}[!t]
    \centering
	\begin{subfigure}[b]{0.24\linewidth}
	\includegraphics[width=\linewidth]{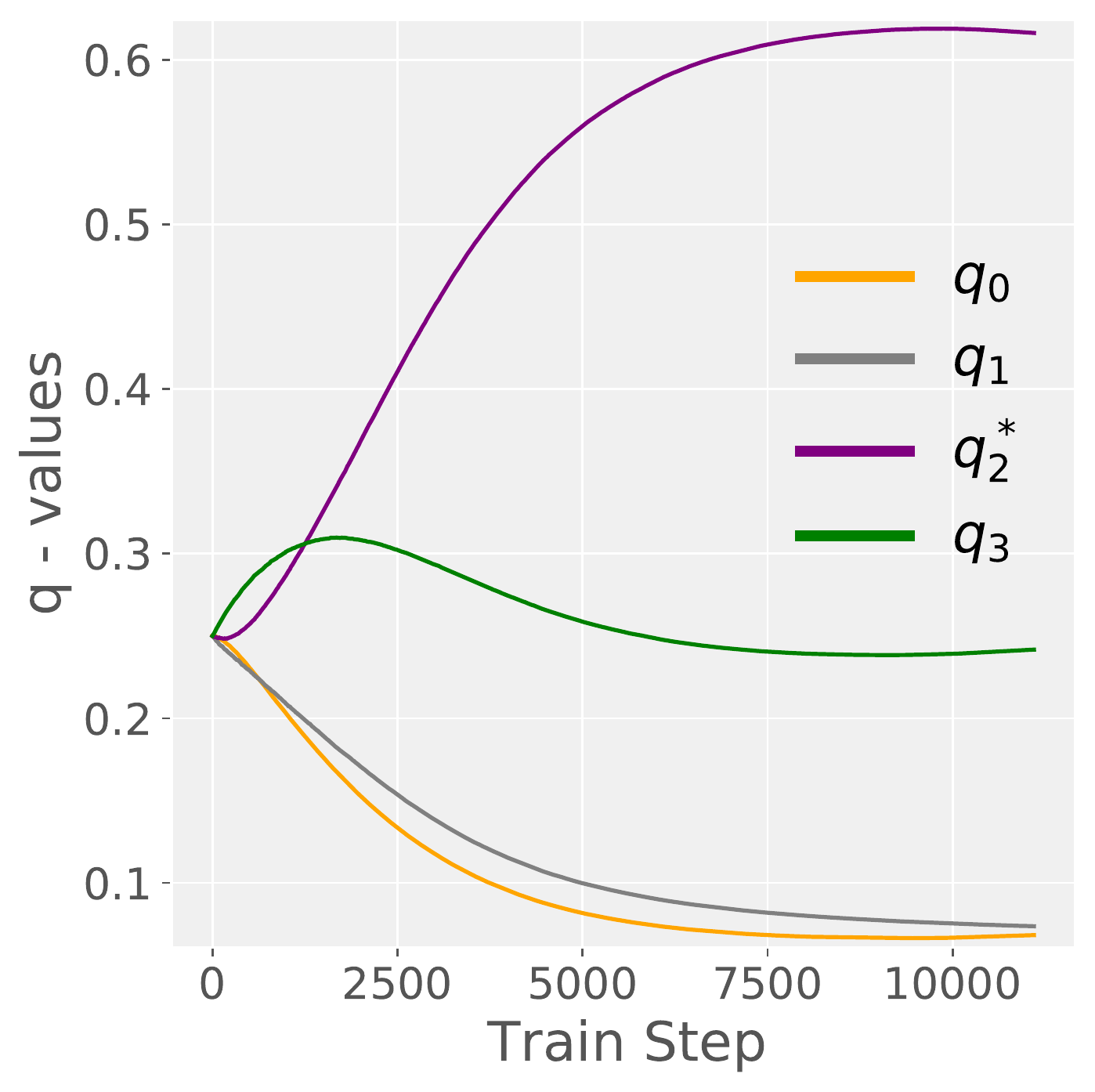}
	\caption{Waterbirds}
	\end{subfigure} 
	\begin{subfigure}[b]{0.24\linewidth}
	\includegraphics[width=\linewidth]{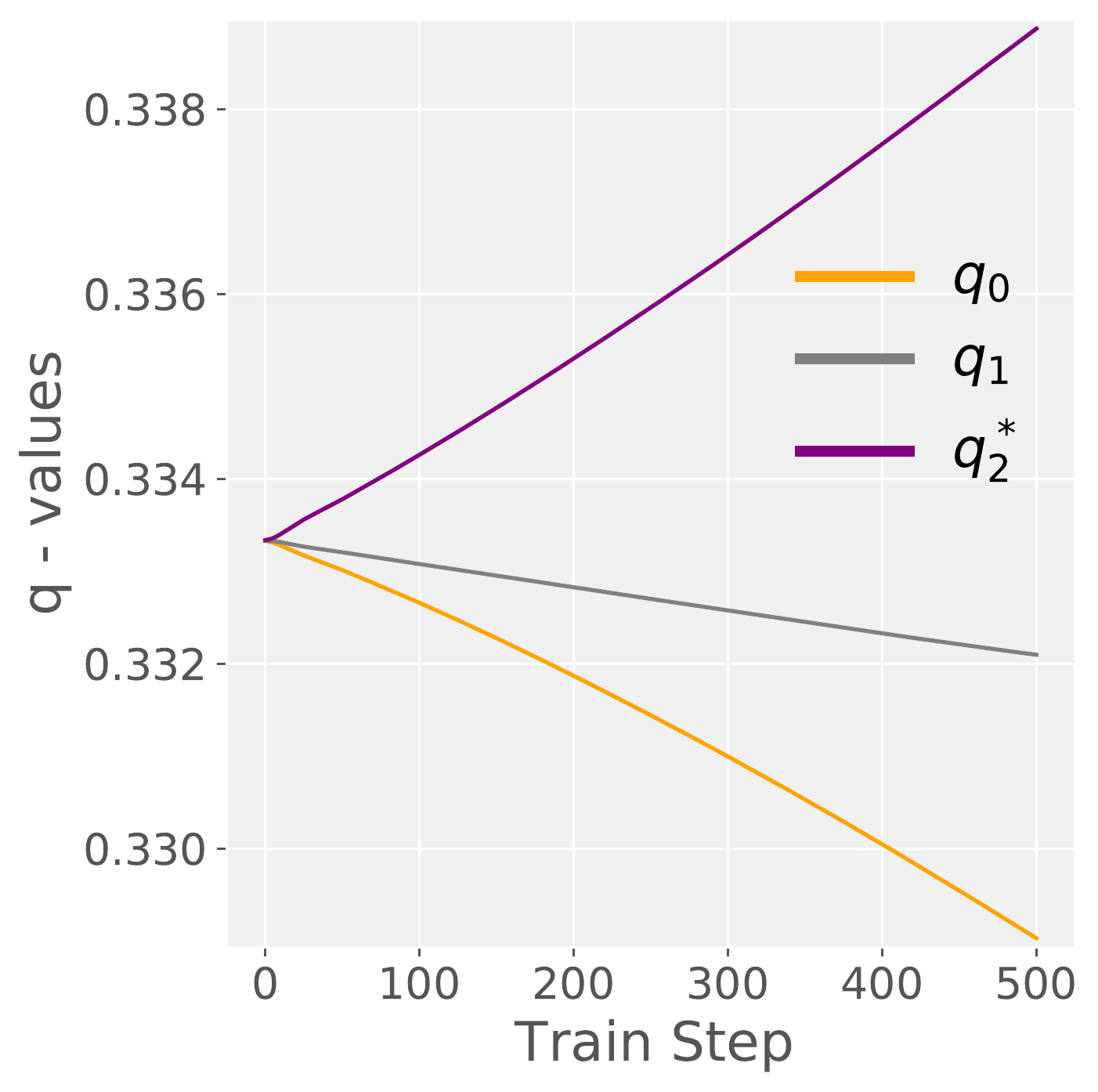}
	\caption{CMNIST}
	\end{subfigure}
	\begin{subfigure}[b]{0.24\linewidth}
	\includegraphics[width=\linewidth]{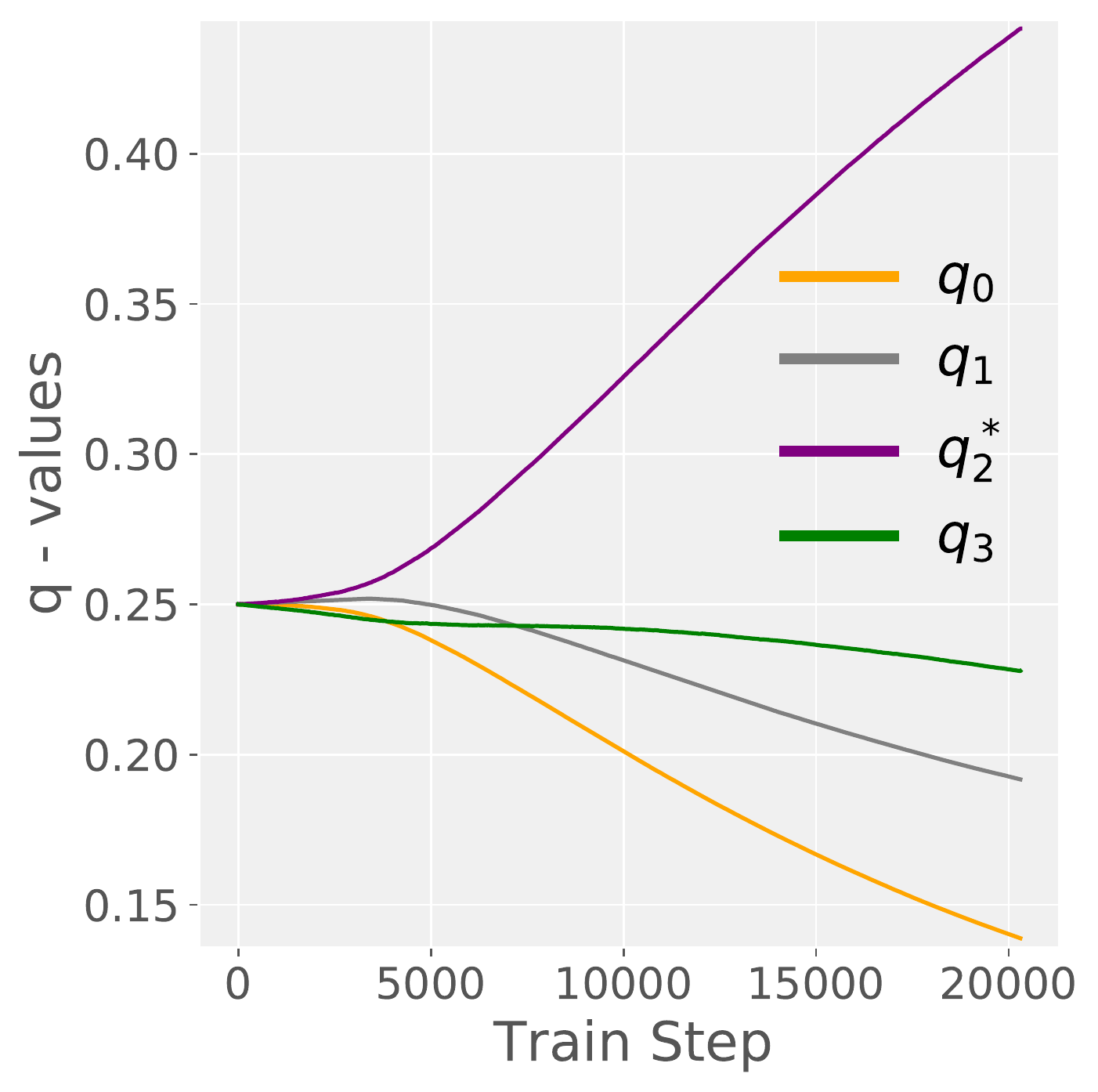}
	\caption{Adult}
	\end{subfigure}
	\begin{subfigure}[b]{0.24\linewidth}
	\includegraphics[width=\linewidth]{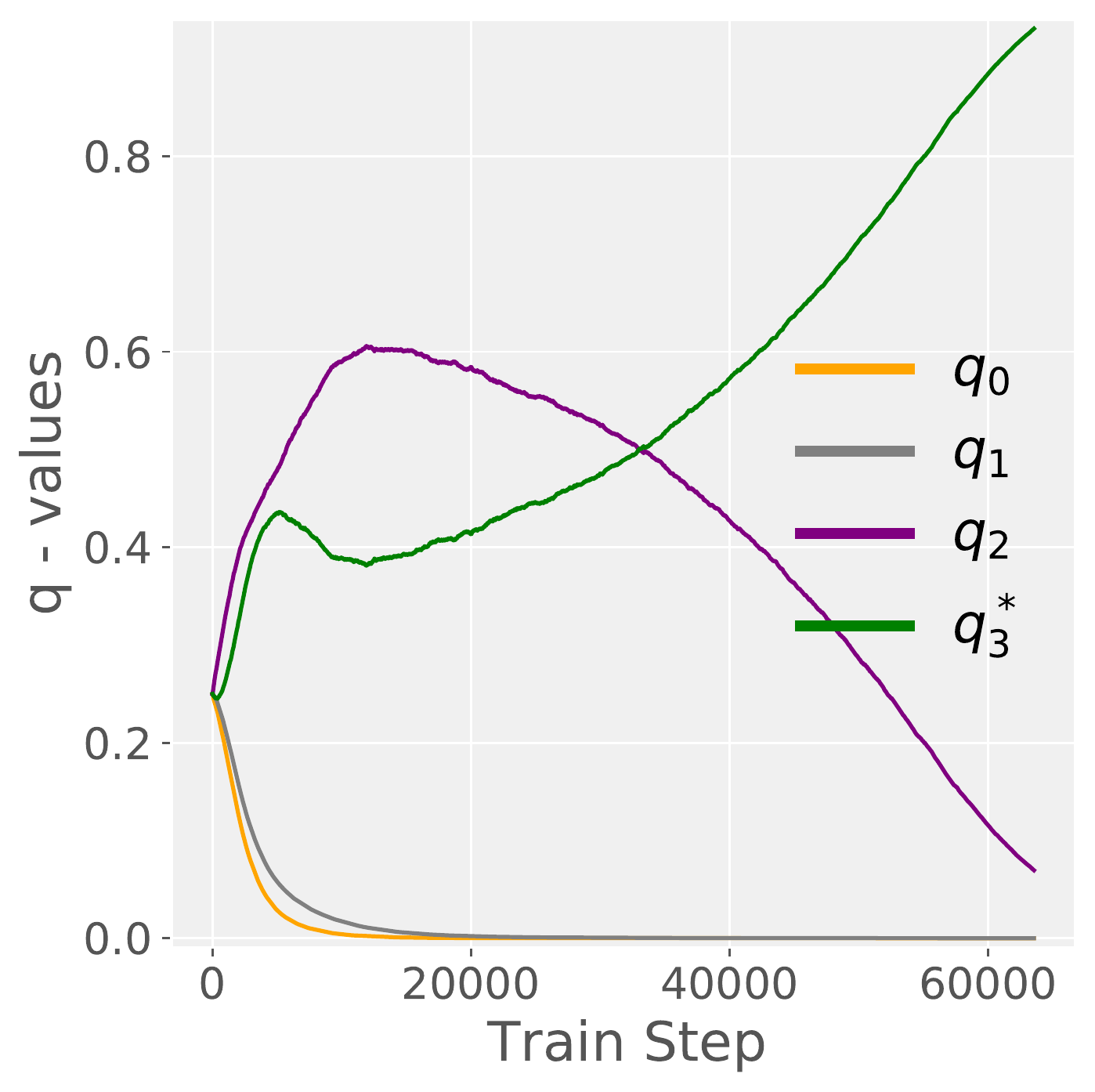}
	\caption{CelebA}
	\end{subfigure}%
	\vspace{-0.05in}
	\caption{ \textbf{Progression of group weights.} The evolution of $q-$values (see Algorithm~\ref{alg:worstoffgropudro_alg}) is plotted for each group. The $q-$values for the minority groups increases gradually while those of the majority groups reduce. A high $q-$value indicates that the corresponding group receives a higher weight relative to other groups. In the plots, minority group is indicate by a $*$ on $q$.}
	\label{fig:weights}
	\vspace{-0.2in}
\end{figure*}

\vspace{-0.05in}
\subsubsection{Progression of Group Weights}
\label{sec:exp_ablation_group_weights}
\vspace{-0.05in}

Our algorithm~\ref{alg:worstoffgropudro_alg}, proceeds by assigning weights, the $q-$value's, to every group. These $q-$values are updated through exponential ascent similar to \citep{sagawa2019distributionally}. Noticeably, the updates on $q$ depend on the worstoff group assignments determined from the constraint set $\mathcal{C}_{\bar{\mathbf{p}},\delta+\epsilon}$. In this section, we investigate on how these group weights evolve through the several iteration of the proposed algorithm. Figure~\ref{fig:weights} plots this evolution across different datasets with the number of curves in a given plot matching the group count of that dataset. Since the $q-$ weights are initialized uniformly at random, the curves begin at the same value. As the training progresses, it is observed that the weights on the minority groups gradually increase and those on the majority groups reduce. The plots indicate a high $q-$ value on the minority groups towards the end of training. This is desirable because the empirical risk on the minority groups getting upweighted relative to the majority groups.


\vspace{-0.05in}
\section{Conclusion}
\label{sec:conclusion}
\vspace{-0.05in}
We present {\worstoffdro}, an invariant learning method across groups when partial group labels are available. 
The formulation of {\worstoffdro} extends that of {\groupdro} by optimizing the loss against the worst-off group assignments in the constraint set. 
By reducing the constraint set with the marginal distribution, we reduce the optimization parameter space while keeping the objective to be an upper bound to that of the {\groupdro} with true group assignments with high probability. 
By harnessing both labeled and unlabeled data in terms of group, we demonstrate in experiments that the {\worstoffdro} outperforms both ERM, UnsupDRO, which do not make use of available group labels, as well as the {\groupdro} (Partial), which does not use unlabeled data.

One future direction, when marginal distribution is not available, is to relax our \textit{missing completely at random} assumption and bring in different but more realistic modeling assumptions on the missingness of group labels, such as missing at random (MAR), where missing values depend on other observed attributes~\citep{fernando2021missing}. 
In addition, it would be valuable to design a reduced constraint set containing the true group assignment to reduce the performance gap to the {\groupdro} (Oracle).

\section*{Ethics Statement}

Machine Learning (ML) models that perform poorly on a minority group or environment have raised a lot of concerns within the AI community and broader society in recent years. 
To democratize ML in real world, learning ML models that perform robustly across groups or environments has become an important venue of research. 
The proposed {\worstoffdro} is a versatile method that could be employed to train an invariant classifier across groups even when the group information is available only for the portion of the data.
This is a rather practical scenario as the group information could be missing for various reasons during the data collection. 
We further emphasize the importance of theoretical result showing the objective of {\worstoffdro} being an upper bound to that of {\groupdro} with complete group information for safety-critical ML applications.

\section*{Reproducibility}

We write our experimental code from scratch using PyTorch library~\citep{paszke2019pytorch}. Due to its similarity, our implementation may closely follow that of {\groupdro}~\citep{sagawa2019distributionally}.\footnote{\url{https://github.com/kohpangwei/group_DRO}} One of the key differentiation of {\worstoffdro} is the inner maximization solver for the worst-off group assignments $\{\hat{g}\}$, which we elaborate the exact code using CVXPY solver~\citep{diamond2016cvxpy} in Algorithm~\ref{alg:code} of Appendix. Additional implementation details, including the neural network architectures, as well as value for hyperparameters including the learning rate, weight decay, batch size, number of training epochs, and algorithm-specific parameters are summarized in Table~\ref{tab:hparams_more} of Appendix and Section~\ref{sec:app_hparam_tuning} and \ref{sec:app_more_details}.

\nocite{*}
\bibliography{ms}

\begin{thebibliography}{48}
\providecommand{\natexlab}[1]{#1}
\providecommand{\url}[1]{\texttt{#1}}
\expandafter\ifx\csname urlstyle\endcsname\relax
  \providecommand{\doi}[1]{doi: #1}\else
  \providecommand{\doi}{doi: \begingroup \urlstyle{rm}\Url}\fi

\bibitem[Agarwal et~al.(2018)Agarwal, Beygelzimer, Dud{\'\i}k, Langford, and
  Wallach]{agarwal2018reductions}
Alekh Agarwal, Alina Beygelzimer, Miroslav Dud{\'\i}k, John Langford, and Hanna
  Wallach.
\newblock A reductions approach to fair classification.
\newblock In \emph{ICML}, 2018.

\bibitem[Arjovsky et~al.(2019)Arjovsky, Bottou, Gulrajani, and
  Lopez-Paz]{arjovsky2019invariant}
Martin Arjovsky, L{\'e}on Bottou, Ishaan Gulrajani, and David Lopez-Paz.
\newblock Invariant risk minimization.
\newblock \emph{arXiv preprint arXiv:1907.02893}, 2019.

\bibitem[Ben-Tal et~al.(2009)Ben-Tal, El~Ghaoui, and Nemirovski]{ben2009robust}
Aharon Ben-Tal, Laurent El~Ghaoui, and Arkadi Nemirovski.
\newblock \emph{Robust optimization}.
\newblock Princeton university press, 2009.

\bibitem[Ben-Tal et~al.(2013)Ben-Tal, Den~Hertog, De~Waegenaere, Melenberg, and
  Rennen]{ben2013robust}
Aharon Ben-Tal, Dick Den~Hertog, Anja De~Waegenaere, Bertrand Melenberg, and
  Gijs Rennen.
\newblock Robust solutions of optimization problems affected by uncertain
  probabilities.
\newblock \emph{Management Science}, 59\penalty0 (2):\penalty0 341--357, 2013.

\bibitem[B{\"o}hlen et~al.(2017)B{\"o}hlen, Chandola, and
  Salunkhe]{bohlen2017server}
Marc B{\"o}hlen, Varun Chandola, and Amol Salunkhe.
\newblock Server, server in the cloud. who is the fairest in the crowd?
\newblock \emph{arXiv preprint arXiv:1711.08801}, 2017.

\bibitem[Choi et~al.(2019)Choi, Gao, Messou, and Huang]{choi2019can}
Jinwoo Choi, Chen Gao, Joseph~CE Messou, and Jia-Bin Huang.
\newblock Why can't i dance in the mall? learning to mitigate scene bias in
  action recognition.
\newblock \emph{NIPS}, 2019.

\bibitem[Creager et~al.(2021)Creager, Jacobsen, and
  Zemel]{creager2021environment}
Elliot Creager, J{\"o}rn-Henrik Jacobsen, and Richard Zemel.
\newblock Environment inference for invariant learning.
\newblock In \emph{ICML}, 2021.

\bibitem[Diamond and Boyd(2016)]{diamond2016cvxpy}
Steven Diamond and Stephen Boyd.
\newblock {CVXPY}: {A} {P}ython-embedded modeling language for convex
  optimization.
\newblock \emph{JMLR}, 17\penalty0 (83):\penalty0 1--5, 2016.

\bibitem[Donini et~al.(2018)Donini, Oneto, Ben-David, Shawe-Taylor, and
  Pontil]{NEURIPS2018_donini}
Michele Donini, Luca Oneto, Shai Ben-David, John~S Shawe-Taylor, and
  Massimiliano Pontil.
\newblock Empirical risk minimization under fairness constraints.
\newblock In S.~Bengio, H.~Wallach, H.~Larochelle, K.~Grauman, N.~Cesa-Bianchi,
  and R.~Garnett, editors, \emph{NIPS}. Curran Associates, Inc., 2018.

\bibitem[Dua et~al.(2017)Dua, Graff, et~al.]{dua2017uci}
Dheeru Dua, Casey Graff, et~al.
\newblock Uci machine learning repository.
\newblock 2017.

\bibitem[Duchi et~al.(2021)Duchi, Glynn, and Namkoong]{duchi2021statistics}
John~C Duchi, Peter~W Glynn, and Hongseok Namkoong.
\newblock Statistics of robust optimization: A generalized empirical likelihood
  approach.
\newblock \emph{Mathematics of Operations Research}, 2021.

\bibitem[Fernando et~al.(2021)Fernando, C{\`e}sar, David, and
  Jos{\'e}]{fernando2021missing}
Mart{\'\i}nez-Plumed Fernando, Ferri C{\`e}sar, Nieves David, and
  Hern{\'a}ndez-Orallo Jos{\'e}.
\newblock Missing the missing values: The ugly duckling of fairness in machine
  learning.
\newblock \emph{International Journal of Intelligent Systems}, 2021.

\bibitem[Gulrajani and Lopez-Paz(2021)]{gulrajani2021in}
Ishaan Gulrajani and David Lopez-Paz.
\newblock In search of lost domain generalization.
\newblock In \emph{ICLR}, 2021.

\bibitem[Hashimoto et~al.(2018)Hashimoto, Srivastava, Namkoong, and
  Liang]{hashimoto2018fairness}
Tatsunori Hashimoto, Megha Srivastava, Hongseok Namkoong, and Percy Liang.
\newblock Fairness without demographics in repeated loss minimization.
\newblock In \emph{International Conference on Machine Learning}, pages
  1929--1938. PMLR, 2018.

\bibitem[Hu et~al.(2018)Hu, Niu, Sato, and Sugiyama]{hu2018does}
Weihua Hu, Gang Niu, Issei Sato, and Masashi Sugiyama.
\newblock Does distributionally robust supervised learning give robust
  classifiers?
\newblock In \emph{ICML}, 2018.

\bibitem[Kehrenberg et~al.(2020)Kehrenberg, Bartlett, Thomas, and
  Quadrianto]{kehrenberg2020null}
Thomas Kehrenberg, Myles Bartlett, Oliver Thomas, and Novi Quadrianto.
\newblock Null-sampling for interpretable and fair representations.
\newblock In \emph{ECCV}, 2020.

\bibitem[Kivinen and Warmuth(1997)]{kivinen1997exponentiated}
Jyrki Kivinen and Manfred~K Warmuth.
\newblock Exponentiated gradient versus gradient descent for linear predictors.
\newblock \emph{Information and computation}, 132\penalty0 (1):\penalty0 1--63,
  1997.

\bibitem[Lahoti et~al.(2020)Lahoti, Beutel, Chen, Lee, Prost, Thain, Wang, and
  Chi]{NEURIPS2020_07fc15c9}
Preethi Lahoti, Alex Beutel, Jilin Chen, Kang Lee, Flavien Prost, Nithum Thain,
  Xuezhi Wang, and Ed~Chi.
\newblock Fairness without demographics through adversarially reweighted
  learning.
\newblock In H.~Larochelle, M.~Ranzato, R.~Hadsell, M.~F. Balcan, and H.~Lin,
  editors, \emph{NeurIPS}, 2020.

\bibitem[LeCun et~al.(1998)LeCun, Bottou, Bengio, and
  Haffner]{lecun1998gradient}
Yann LeCun, L{\'e}on Bottou, Yoshua Bengio, and Patrick Haffner.
\newblock Gradient-based learning applied to document recognition.
\newblock \emph{Proceedings of the IEEE}, 86\penalty0 (11):\penalty0
  2278--2324, 1998.

\bibitem[Levy et~al.(2020)Levy, Carmon, Duchi, and Sidford]{levy2020large}
Daniel Levy, Yair Carmon, John~C Duchi, and Aaron Sidford.
\newblock Large-scale methods for distributionally robust optimization.
\newblock \emph{arXiv preprint arXiv:2010.05893}, 2020.

\bibitem[Liu et~al.(2021)Liu, Haghgoo, Chen, Raghunathan, Koh, Sagawa, Liang,
  and Finn]{liu2021just}
Evan~Z Liu, Behzad Haghgoo, Annie~S Chen, Aditi Raghunathan, Pang~Wei Koh,
  Shiori Sagawa, Percy Liang, and Chelsea Finn.
\newblock Just train twice: Improving group robustness without training group
  information.
\newblock In \emph{ICML}, 2021.

\bibitem[Liu et~al.(2015)Liu, Luo, Wang, and Tang]{liu2015deep}
Ziwei Liu, Ping Luo, Xiaogang Wang, and Xiaoou Tang.
\newblock Deep learning face attributes in the wild.
\newblock In \emph{ICCV}, 2015.

\bibitem[Mahajan et~al.(2021)Mahajan, Tople, and Sharma]{mahajan2021domain}
Divyat Mahajan, Shruti Tople, and Amit Sharma.
\newblock Domain generalization using causal matching.
\newblock In \emph{International Conference on Machine Learning}, pages
  7313--7324. PMLR, 2021.

\bibitem[Mohan and Pearl(2014)]{mohan2014graphical}
Karthika Mohan and Judea Pearl.
\newblock Graphical models for recovering probabilistic and causal queries from
  missing data.
\newblock \emph{Advances in Neural Information Processing Systems},
  27:\penalty0 1520--1528, 2014.

\bibitem[Mohri et~al.(2019)Mohri, Sivek, and Suresh]{mohri2019agnostic}
Mehryar Mohri, Gary Sivek, and Ananda~Theertha Suresh.
\newblock Agnostic federated learning.
\newblock In \emph{ICML}, 2019.

\bibitem[Moyer et~al.(2018)Moyer, Gao, Brekelmans, Galstyan, and
  Ver~Steeg]{NEURIPS2018_invrep_noadv}
Daniel Moyer, Shuyang Gao, Rob Brekelmans, Aram Galstyan, and Greg Ver~Steeg.
\newblock Invariant representations without adversarial training.
\newblock In S.~Bengio, H.~Wallach, H.~Larochelle, K.~Grauman, N.~Cesa-Bianchi,
  and R.~Garnett, editors, \emph{Advances in Neural Information Processing
  Systems}, volume~31. Curran Associates, Inc., 2018.
\newblock URL
  \url{https://proceedings.neurips.cc/paper/2018/file/415185ea244ea2b2bedeb0449b926802-Paper.pdf}.

\bibitem[Namkoong and Duchi(2016)]{namkoong2016stochastic}
Hongseok Namkoong and John~C Duchi.
\newblock Stochastic gradient methods for distributionally robust optimization
  with f-divergences.
\newblock In \emph{NIPS}, 2016.

\bibitem[Oakden-Rayner et~al.(2020)Oakden-Rayner, Dunnmon, Carneiro, and
  R{\'e}]{oakden2020hidden}
Luke Oakden-Rayner, Jared Dunnmon, Gustavo Carneiro, and Christopher R{\'e}.
\newblock Hidden stratification causes clinically meaningful failures in
  machine learning for medical imaging.
\newblock In \emph{Proc ACM conference on health, inference, and learning},
  2020.

\bibitem[Oren et~al.(2019)Oren, Sagawa, Hashimoto, and
  Liang]{oren2019distributionally}
Yonatan Oren, Shiori Sagawa, Tatsunori~B Hashimoto, and Percy Liang.
\newblock Distributionally robust language modeling.
\newblock In \emph{EMNLP/IJCNLP}, 2019.

\bibitem[Paszke et~al.(2019)Paszke, Gross, Massa, Lerer, Bradbury, Chanan,
  Killeen, Lin, Gimelshein, Antiga, et~al.]{paszke2019pytorch}
Adam Paszke, Sam Gross, Francisco Massa, Adam Lerer, James Bradbury, Gregory
  Chanan, Trevor Killeen, Zeming Lin, Natalia Gimelshein, Luca Antiga, et~al.
\newblock Pytorch: An imperative style, high-performance deep learning library.
\newblock \emph{NIPS}, 2019.

\bibitem[Rahimian and Mehrotra(2019)]{rahimian2019distributionally}
Hamed Rahimian and Sanjay Mehrotra.
\newblock Distributionally robust optimization: A review.
\newblock \emph{arXiv preprint arXiv:1908.05659}, 2019.

\bibitem[Rawls(2001)]{rawls2001justice}
John Rawls.
\newblock \emph{Justice as fairness: A restatement}.
\newblock Harvard University Press, 2001.

\bibitem[Roddenberry et~al.(2021)Roddenberry, Frantzen, Schaub, and
  Segarra]{roddenberry2021hodgelets}
T~Mitchell Roddenberry, Florian Frantzen, Michael~T Schaub, and Santiago
  Segarra.
\newblock Hodgelets: Localized spectral representations of flows on simplicial
  complexes.
\newblock \emph{arXiv preprint arXiv:2109.08728}, 2021.

\bibitem[Rubin(1976)]{rubin1976inference}
Donald~B Rubin.
\newblock Inference and missing data.
\newblock \emph{Biometrika}, 63\penalty0 (3):\penalty0 581--592, 1976.

\bibitem[Sagawa et~al.(2019)Sagawa, Koh, Hashimoto, and
  Liang]{sagawa2019distributionally}
Shiori Sagawa, Pang~Wei Koh, Tatsunori~B Hashimoto, and Percy Liang.
\newblock Distributionally robust neural networks.
\newblock In \emph{ICLR}, 2019.

\bibitem[Sagawa et~al.(2020)Sagawa, Raghunathan, Koh, and
  Liang]{sagawa2020investigation}
Shiori Sagawa, Aditi Raghunathan, Pang~Wei Koh, and Percy Liang.
\newblock An investigation of why overparameterization exacerbates spurious
  correlations.
\newblock In \emph{ICML}, 2020.

\bibitem[Shapiro et~al.(2021)Shapiro, Dentcheva, and
  Ruszczynski]{shapiro2021lectures}
Alexander Shapiro, Darinka Dentcheva, and Andrzej Ruszczynski.
\newblock \emph{Lectures on stochastic programming: modeling and theory}.
\newblock SIAM, 2021.

\bibitem[Sohoni et~al.(2020)Sohoni, Dunnmon, Angus, Gu, and
  R\'{e}]{NEURIPS2020_e0688d13}
Nimit Sohoni, Jared Dunnmon, Geoffrey Angus, Albert Gu, and Christopher R\'{e}.
\newblock No subclass left behind: Fine-grained robustness in coarse-grained
  classification problems.
\newblock In H.~Larochelle, M.~Ranzato, R.~Hadsell, M.~F. Balcan, and H.~Lin,
  editors, \emph{NeurIPS}, 2020.

\bibitem[Wah et~al.(2011)Wah, Branson, Welinder, Perona, and
  Belongie]{wah2011caltech}
Catherine Wah, Steve Branson, Peter Welinder, Pietro Perona, and Serge
  Belongie.
\newblock The caltech-ucsd birds-200-2011 dataset.
\newblock 2011.

\bibitem[Wang et~al.(2020)Wang, Guo, Narasimhan, Cotter, Gupta, and
  Jordan]{NEURIPS2020_noisy}
Serena Wang, Wenshuo Guo, Harikrishna Narasimhan, Andrew Cotter, Maya Gupta,
  and Michael Jordan.
\newblock Robust optimization for fairness with noisy protected groups.
\newblock In H.~Larochelle, M.~Ranzato, R.~Hadsell, M.~F. Balcan, and H.~Lin,
  editors, \emph{NeurIPS}, 2020.

\bibitem[Wasserman(2004)]{wasserman2004all}
Larry Wasserman.
\newblock \emph{All of statistics: a concise course in statistical inference},
  volume~26.
\newblock Springer, 2004.

\bibitem[Welinder et~al.(2010)Welinder, Branson, Mita, Wah, Schroff, Belongie,
  and Perona]{welinder2010caltech}
P~Welinder, S~Branson, T~Mita, C~Wah, F~Schroff, S~Belongie, and P~Perona.
\newblock Caltech-ucsd birds 200. technical report cns-tr-2010-001.
\newblock \emph{California Institute of Technology}, 2010.

\bibitem[Xie et~al.(2017)Xie, Dai, Du, Hovy, and Neubig]{xie2017controllable}
Qizhe Xie, Zihang Dai, Yulun Du, Eduard Hovy, and Graham Neubig.
\newblock Controllable invariance through adversarial feature learning.
\newblock \emph{arXiv preprint arXiv:1705.11122}, 2017.

\bibitem[Zhang and Shah(2014)]{zhang2014fairness}
Chongjie Zhang and Julie~A Shah.
\newblock Fairness in multi-agent sequential decision-making.
\newblock In \emph{NIPS}, 2014.

\bibitem[Zhang et~al.(2020)Zhang, Menon, Veit, Bhojanapalli, Kumar, and
  Sra]{zhang2020coping}
Jingzhao Zhang, Aditya~Krishna Menon, Andreas Veit, Srinadh Bhojanapalli,
  Sanjiv Kumar, and Suvrit Sra.
\newblock Coping with label shift via distributionally robust optimisation.
\newblock In \emph{ICLR}, 2020.

\bibitem[Zhao and Udell(2020)]{NEURIPS2020_invprop_udell}
Yuxuan Zhao and Madeleine Udell.
\newblock Matrix completion with quantified uncertainty through low rank
  gaussian copula.
\newblock In H.~Larochelle, M.~Ranzato, R.~Hadsell, M.~F. Balcan, and H.~Lin,
  editors, \emph{Advances in Neural Information Processing Systems}, volume~33,
  pages 20977--20988. Curran Associates, Inc., 2020.
\newblock URL
  \url{https://proceedings.neurips.cc/paper/2020/file/f076073b2082f8741a9cd07b789c77a0-Paper.pdf}.

\bibitem[Zhou et~al.(2017)Zhou, Lapedriza, Khosla, Oliva, and
  Torralba]{zhou2017places}
Bolei Zhou, Agata Lapedriza, Aditya Khosla, Aude Oliva, and Antonio Torralba.
\newblock Places: A 10 million image database for scene recognition.
\newblock \emph{PAMI}, 40\penalty0 (6):\penalty0 1452--1464, 2017.

\bibitem[Zhu and Goldberg(2009)]{zhu2009introduction}
Xiaojin Zhu and Andrew~B Goldberg.
\newblock Introduction to semi-supervised learning.
\newblock \emph{Synthesis lectures on artificial intelligence and machine
  learning}, 3\penalty0 (1):\penalty0 1--130, 2009.

\end{thebibliography}
\bibliographystyle{plainnat}

\clearpage
\appendix
\section{Appendix}

\subsection{Proof of Lemmas}
\label{sec:proof_of_lemma}

\setcounter{theorem}{0}
\begin{lemma}
    Denote $\mathcal{L}_{\mathrm{GDRO}}$ at a given $w$ and $q$ parameters as $\mathcal{L}_{\mathrm{GDRO}(w, q)}$. Similarly $\mathcal{L}_{\mathrm{WDRO}}(\mathcal{C})$ at a fixed $w$ and $q$ as $\mathcal{L}_{\mathrm{WDRO}(w, q)}(\mathcal{C})$. When the ground-truth group assignment $\{g^{\star}_{i}\}_{i=1}^{N} \in \mathcal{C}$, we have
    \begin{align}
        \min_{w}\max_{q\in\Delta^{M}}\mathcal{L}_{\mathrm{GDRO}(w, q)} \le \min_{w}\max_{q\in\Delta^{M}}\mathcal{L}_{\mathrm{WDRO}(w, q)}(\mathcal{C})
    \end{align}
\end{lemma}
\begin{proof}
    Under the case $\{g^{\star}_{i}\}_{i=1}^{N} \in \mathcal{C}$, due to the $\max$ over $\mathcal{C}$, we have
    \begin{align}
        \label{eq:ub_proof_one}
        \mathcal{L}_{\mathrm{GDRO}(w, q)} \le \mathcal{L}_{\mathrm{WDRO}(w, q)}(\mathcal{C}) \quad \forall w, q
    \end{align}
    Define $q^*_\mathrm{WDRO} = {\argmax}_{q\in\Delta^{M}} \mathcal{L}_{\mathrm{WDRO}(w, q)}(\mathcal{C})$ and $q^*_\mathrm{GDRO} = \argmax_{q\in\Delta^{M}} \mathcal{L}_{\mathrm{GDRO}(w, q)}$. \\
    From the above definitions, we have,
    \begin{align}
        \label{eq:ub_proof_two}
        \mathcal{L}_{\mathrm{WDRO}(w, q)}(\mathcal{C}) \le \mathcal{L}_{\mathrm{WDRO}(w, q^*_\mathrm{WDRO})}(\mathcal{C})
    \end{align}
    Moreover, 
    \begin{align}
        \mathcal{L}_{\mathrm{GDRO}(w, q^*_\mathrm{GDRO})} &\le \mathcal{L}_{\mathrm{WDRO}(w, q^*_\mathrm{GDRO})}(\mathcal{C}) \quad \text{from } \eqref{eq:ub_proof_one} \\
        \mathcal{L}_{\mathrm{WDRO}(w, q^*_\mathrm{GDRO})}(\mathcal{C}) &\le \mathcal{L}_{\mathrm{WDRO}(w, q^*_\mathrm{WDRO})}(\mathcal{C}) \quad \text{from } \eqref{eq:ub_proof_two} \\
        \label{eq:ub_proof_optim}
        \implies \mathcal{L}_{\mathrm{GDRO}(w, q^*_\mathrm{GDRO})} &\le \mathcal{L}_{\mathrm{WDRO}(w, q^*_\mathrm{WDRO})}(\mathcal{C})
    \end{align}
    Minimizing \eqref{eq:ub_proof_optim} over $w$, we obtain,
    \begin{align*}
        \min_{w}\max_{q\in\Delta^{M}}\mathcal{L}_{\mathrm{GDRO}(w, q)} \le \min_{w}\max_{q\in\Delta^{M}}\mathcal{L}_{\mathrm{WDRO}(w, q)}(\mathcal{C})
    \end{align*}
\end{proof}

\begin{lemma}
The constraint set $\mathcal{C}_{\mathbf{p}^{\star}, \epsilon}$ contains the true group labels $\{g_{i}^{\star}\}_{i=1}^{N}$ with high probability:
\begin{equation}
    P(\{g_{i}^{\star}\}_{i=1}^{N}\,{\in}\,\mathcal{C}_{\mathbf{p}^{\star}, \epsilon}) \geq 1\,{-}\,2e^{-2N\epsilon^{2}}\nonumber
\end{equation}
\end{lemma}
\begin{proof}
The probability of the true group assignment $\{g_{i}^{\star}\}_{i=1}^{N}$ in the constraint set $\mathcal{C}_{\mathbf{p}^{\star}, \epsilon}$ is written as follows:
\begin{equation}
    P(\{g_{i}^{\star}\}_{i=1}^{N}\,{\in}\,\mathcal{C}_{\mathbf{p}^{\star}, \epsilon}) = P\Big(\big\vert p_{j}^{\star}\,{-}\,\frac{1}{N}\sum\nolimits_{i=1}^{N}\mathbbm{1}\{g_{i}^{\star}=j\}\big\vert\,{\leq}\,\epsilon\Big) \geq 1 - 2e^{-2N\epsilon^2}\label{eq:groupdro_constraint_set_with_marginal_proof}
\end{equation}
where \eqref{eq:groupdro_constraint_set_with_marginal_proof} holds true from the Hoeffding's inequality.
\end{proof}

\begin{lemma}
The constraint set $\mathcal{C}_{\bar{\mathbf{p}}, \delta+\epsilon}$ contains the true group labels $\{g^{\star}_{i}\}_{i=1}^{N}$ with high probability:
\begin{equation}
    P(\{g_{i}^{\star}\}_{i=1}^{N}\,{\in}\,\mathcal{C}_{\bar{\mathbf{p}}, \delta+\epsilon}) \geq 1\,{-}\,2e^{-2N\epsilon^{2}}\,{-}\,2e^{-2K\delta^{2}}\nonumber
\end{equation}
\end{lemma}

\begin{proof}
    Using Hoeffding's inequality, we can show that the estimation error of the marginal distribution is bounded by $\delta$ with high probability as follows:
    \begin{equation}
        P(\vert p_{j}^{\star} - \bar{p}_{j}\vert \leq \delta) = P\Big(\big\vert p_{j}^{\star} - \frac{1}{K}\sum\nolimits_{i=1}^{K}\mathbbm{1}\{g_{i}^{\star}=j\}\big\vert \leq \delta\Big) \geq 1-2e^{-2K\delta^2}
    \end{equation}
    Furthermore, we show using Hoeffding's inequality that
    \begin{equation}
        P\Big(\big\vert p_{j}^{\star} - \frac{1}{N}\sum\nolimits_{i=1}^{N}\mathbbm{1}\{g_{i}=j\}\big\vert \leq \epsilon\Big) \geq 1-2e^{-2N\epsilon^2}
    \end{equation}
    Now, the probability of the true group assignment $\{g_{i}^{\star}\}_{i=1}^{N}$ in the constraint set $\mathcal{C}_{\bar{\mathbf{p}}, \delta+\epsilon}$ is written as follows:
    \begin{align}
        P(\{g_{i}^{\star}\}_{i=1}^{N}\,{\in}\,\mathcal{C}_{\bar{\mathbf{p}}, \delta+\epsilon}) &= P\Big(\big\vert \bar{p}_{j}\,{-}\,\frac{1}{N}\sum\nolimits_{i=1}^{N}\mathbbm{1}\{g_{i}=j\}\big\vert\,{\leq}\,\delta+\epsilon\Big) \label{eq:groupdro_constraint_set_mcar_proof_1}\\
        & \geq P\Big(\Big\{\big\vert p_{j}^{\star} - \bar{p}_{j}\big\vert \leq \delta\Big\}\cap\Big\{\big\vert p_{j}^{\star} - \frac{1}{N}\sum\nolimits_{i=1}^{N}\mathbbm{1}\{g_{i}=j\}\big\vert \leq \epsilon\Big\}\Big) \label{eq:groupdro_constraint_set_mcar_proof_2}\\
        & \geq P\Big(\big\vert p_{j}^{\star} - \bar{p}_{j}\big\vert \leq \delta\Big) + P\Big(\big\vert p_{j}^{\star} - \frac{1}{N}\sum\nolimits_{i=1}^{N}\mathbbm{1}\{g_{i}=j\}\big\vert \leq \epsilon\Big) - 1 \label{eq:groupdro_constraint_set_mcar_proof_3}\\
        & \geq 1-2e^{-2K\delta^2}-2e^{-2N\epsilon^2}
    \end{align}
    where \eqref{eq:groupdro_constraint_set_mcar_proof_2} is due to that the intersection of events in \eqref{eq:groupdro_constraint_set_mcar_proof_2} is a subset of an event in \eqref{eq:groupdro_constraint_set_mcar_proof_1}, and \eqref{eq:groupdro_constraint_set_mcar_proof_3} is derived using union bound. 
\end{proof}

\begin{figure*}[t!]
	\begin{subfigure}[b]{0.24\linewidth}
	\includegraphics[width=\linewidth]{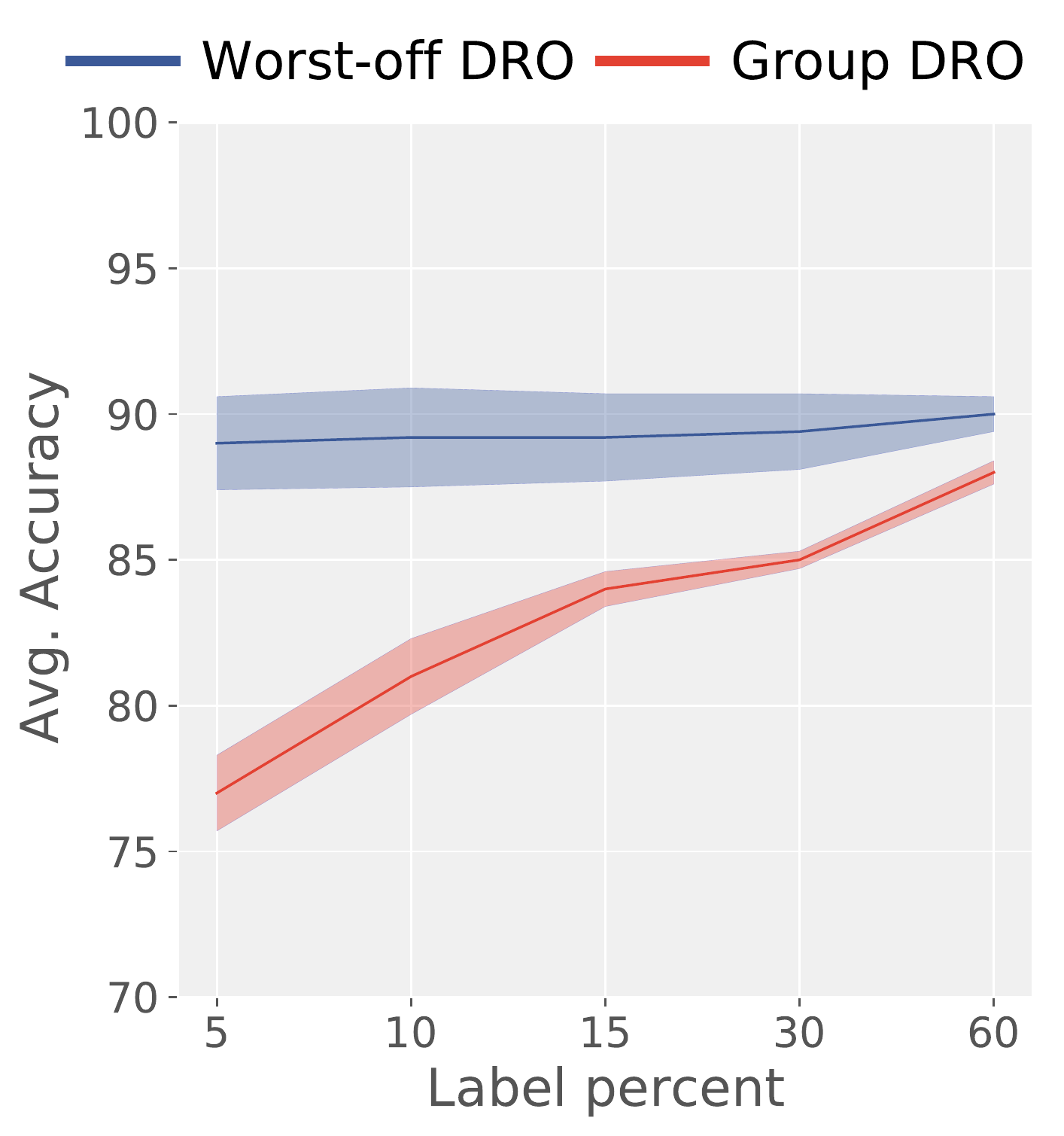}
	\caption{Waterbirds}
	\end{subfigure} 
	\begin{subfigure}[b]{0.24\linewidth}
	\includegraphics[width=\linewidth]{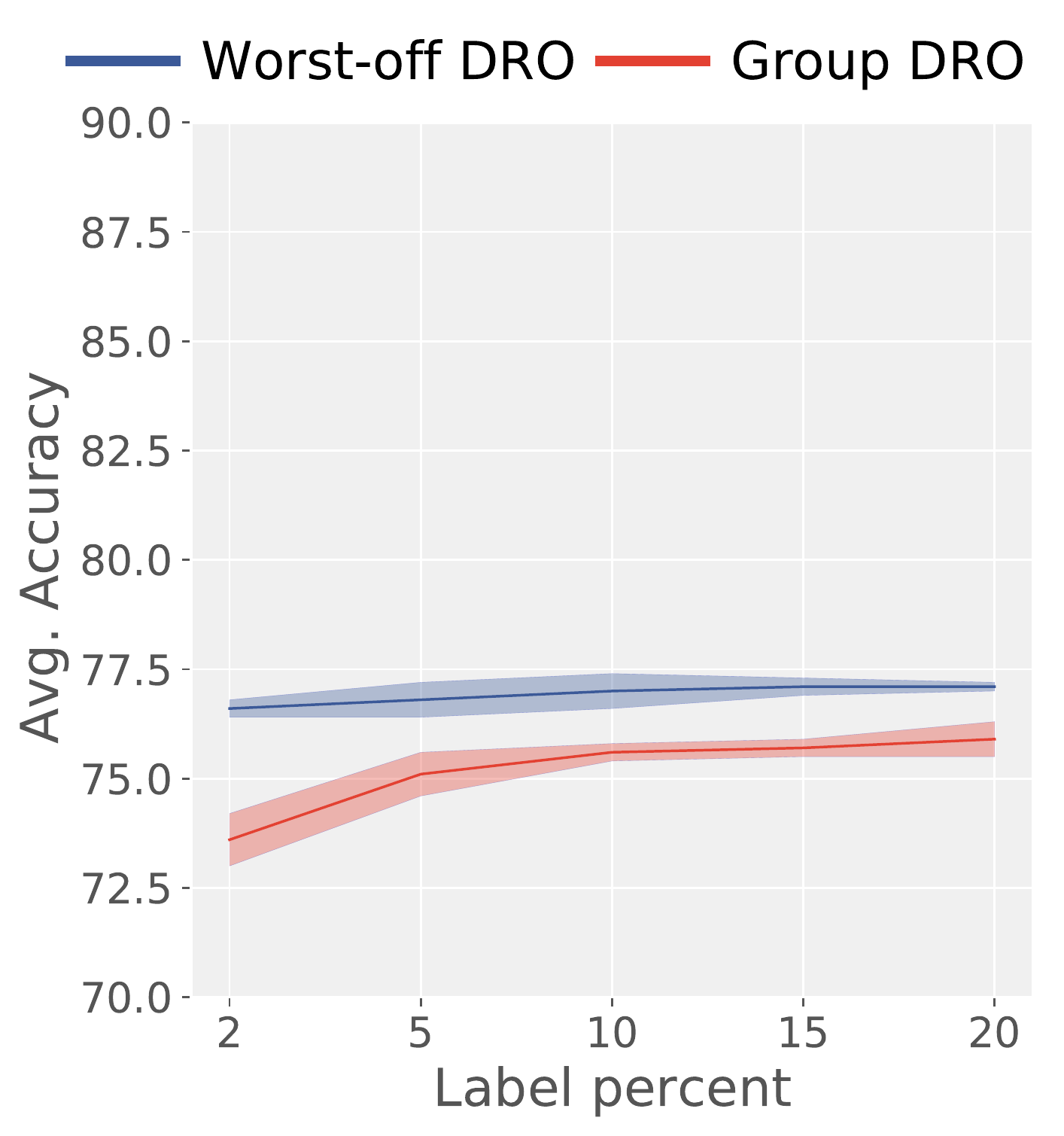}
	\caption{CMNIST}
	\end{subfigure}
	\begin{subfigure}[b]{0.24\linewidth}
	\includegraphics[width=\linewidth]{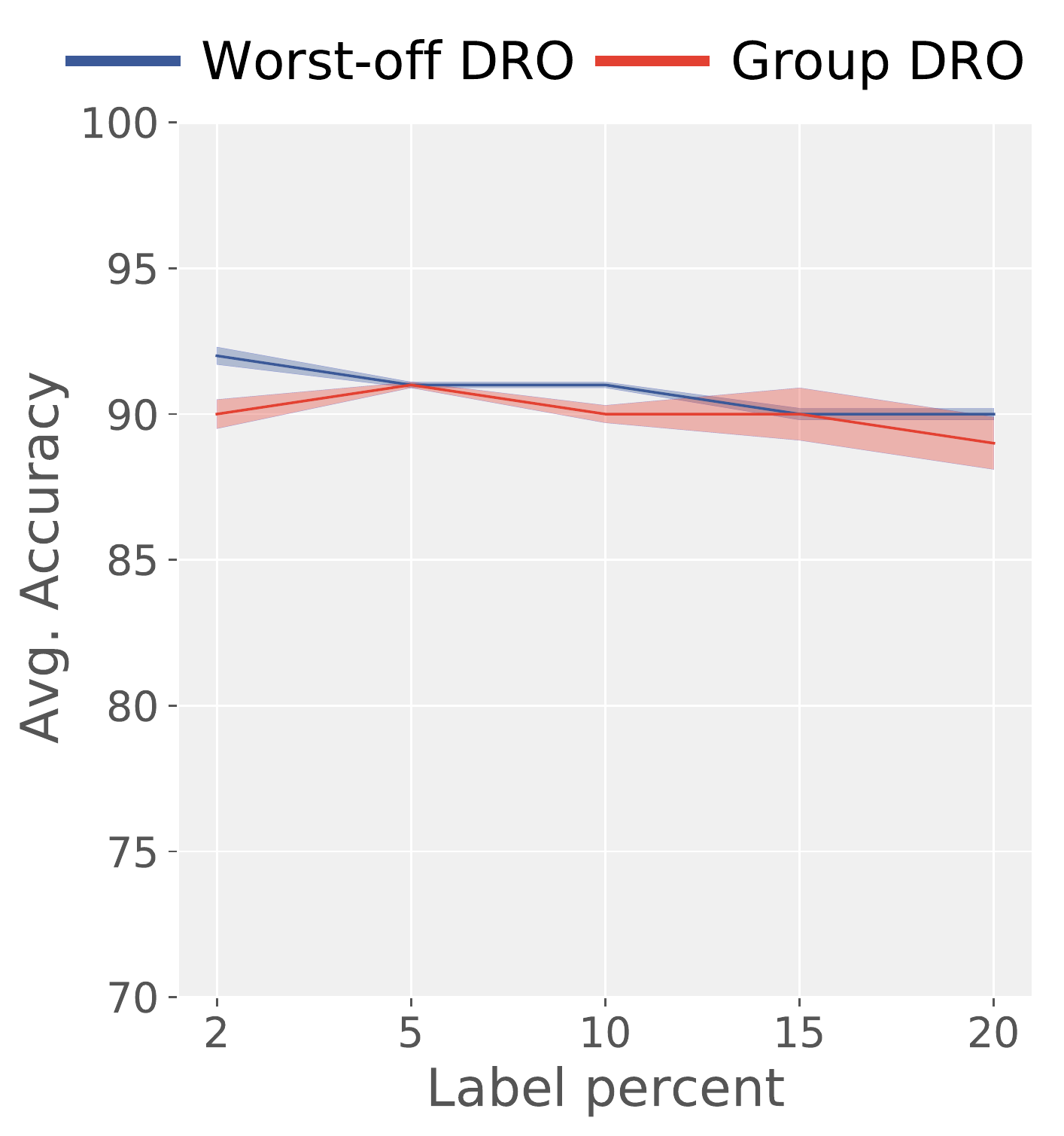}
	\caption{Adult}
	\end{subfigure}
	\begin{subfigure}[b]{0.24\linewidth}
	\includegraphics[width=\linewidth]{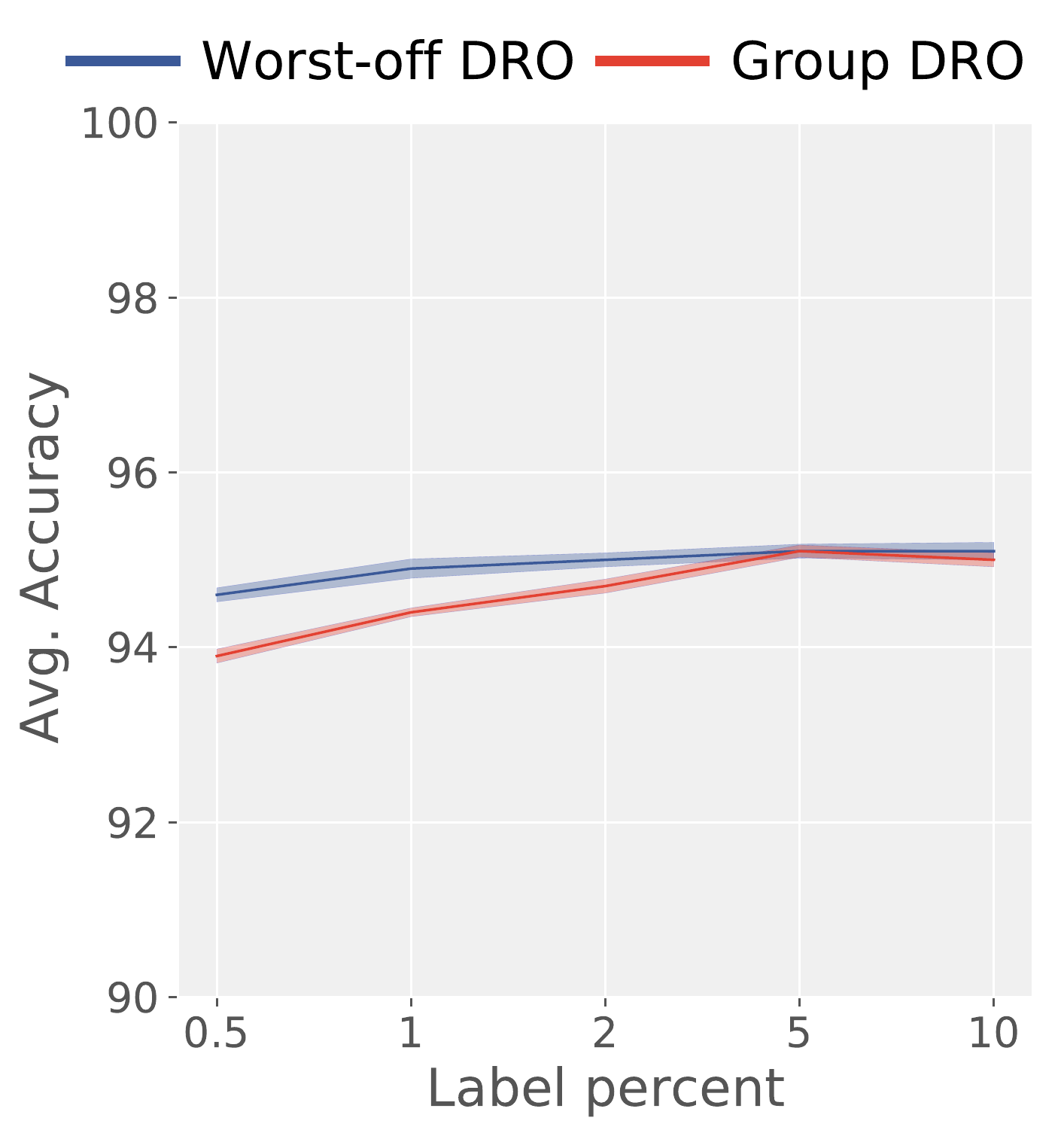}
	\caption{CelebA}
	\end{subfigure}%
	\centering
	\caption{ \textbf{Increasing the labelled samples - Average Group Accuracy.} We plot the average group accuracies as a function of labelled samples. These accuracies remain fairly similar as the count of labelled samples grows.  }
	\label{fig:labpercent_average}
\end{figure*}

\begin{table}[!t]
\footnotesize
\centering
\begin{tabular}{cc|cc}
\toprule
\multicolumn{2}{c}{Waterbirds}                                          & \multicolumn{2}{c}{CMNIST}                                                 \\\midrule\midrule
Learning Rate                    & $0.0001, 0.00001, 0.000001$            & Learning Rate                    & $0.001, 0.0001, 0.00001$                  \\
Weight Decay                     & $1.5, 1.0 , 0.1$                       & Weight Decay                     & $0.01, 0.001, 0.0001$                     \\
$\eta_\text{UDRO}$                     & $0.9, 0.8, 0.7, 0.6, 0.5, 0.4, 0.3$    & $\eta_\text{UDRO}$                     & $0.9, 0.8, 0.7, 0.6, 0.5, 0.4, 0.3$       \\
$\eta_\text{GDRO}$                     & $0.1, 0.01, 0.001$                     & $\eta_\text{GDRO}$ & $0.01, 0.001, 0.0001$ \\
$\eta_\text{WDRO}$                     & $0.1, 0.01, 0.001$                     & $\eta_\text{WDRO}$ & $0.01, 0.001, 0.0001$ \\\bottomrule
\\
\toprule
\multicolumn{2}{c}{Adult}                                               & \multicolumn{2}{c}{CelebA}                                                 \\\midrule\midrule
Learning Rate                    & $0.001, 0.0001, 0.00001$               & Learning Rate                    & $0.0001, 0.00001, 0.000001$               \\
Weight Decay                     & $0.01, 0.001, 0.0001$                  & Weight Decay                     & $1.0, 0.1, 0.01$                          \\
$\eta_\text{UDRO}$                     & $0.9, 0.8, 0.7, 0.6, 0.5, 0.4, 0.3$    & $\eta_\text{UDRO}$                     & $0.9, 0.8, 0.7, 0.6, 0.5, 0.4, 0.3$       \\
$\eta_\text{GDRO}$                     & $0.01, 0.001, 0.0001$                  & $\eta_\text{GDRO}$                     & $0.1, 0.01, 0.001$                        \\
$\eta_\text{WDRO}$                     & $0.01, 0.001, 0.0001$                  & $\eta_\text{WDRO}$                     & $0.1, 0.01, 0.001$                       \\\bottomrule
\end{tabular}
	\caption{\footnotesize \label{tab:hparams_more} \textbf{Grid search for Table~\ref{tab:quant_results}.} The range of values for each hyper-parameter is listed. A grid search over these hyper-parameters is conducted to identify the best performing model. Models outside these range values were observed to be either unstable or not converging. Model selection is done based on NVP (novel validation procedure) where first the models, with higher overall accuracies, are selected. From the top five such performing models, the one with the highest minority group accuracy is picked.}
\end{table}

\begin{table}[!t]
	\centering
	\resizebox{\columnwidth}{!}{%
		\begin{tabular}{c c c c c c c c}
		\toprule
			Dataset    & Method                  & Architecture  &Learning Rate &  Weight Decay &  Batch Size & \# Epochs & Other params \\ \midrule\midrule
			Waterbirds & ERM                     &  ResNet50     & 0.0001  & 0.1    & 128 & 300 & -  \\
			Waterbirds & {\unsupdro}                &  ResNet50     & 0.0001  & 0.1    & 128 & 300 & $\eta$=0.3 \\
			Waterbirds & {\groupdro}-(Oracle)       &  ResNet50     & 0.00001 & 1.0    & 128 & 300 & $\eta$=0.001 \\
			Waterbirds & {\groupdro}-(Partial)      &  ResNet50     & 0.00001 & 0.1    & 128 & 300 & $\eta$=0.001 \\
			Waterbirds & {\worstoffdro}             &  ResNet50     & 0.00001 & 1.0    & 128 & 300 & $\eta$=0.001 \\
			CMNIST     & ERM                     &  MLP(390,390) & 0.001   & 0.01  & - & 500 & -  \\
			CMNIST     & {\unsupdro}                &  MLP(390,390) & 0.00001 & 0.001  & - & 500 & $\eta$=0.4  \\
			CMNIST     & {\groupdro}-(Oracle)       &  MLP(390,390) & 0.0001  & 0.001  & - & 500 & $\eta$=0.001\\
			CMNIST     & {\groupdro}-(Partial)      &  MLP(390,390) & 0.001  & 0.01  & - & 500 & $\eta$=0.001\\
			CMNIST     & {\worstoffdro}             &  MLP(390,390) & 0.0001  & 0.01  & - & 500 & $\eta$=0.0001  \\ 
			Adult      & ERM                     &  MLP(64,32)   &  0.0001 & 0.001   & 128 & 200 & - \\
			Adult      & {\unsupdro}                &  MLP(64,32)   &  0.0001 & 0.001   & 128 & 200 & $\eta$=0.3  \\
			Adult      & {\groupdro}-(Oracle)       &  MLP(64,32)   &  0.0001 & 0.001   & 128 & 200 & $\eta$=0.0001  \\			Adult      & {\groupdro}-(Partial)      &  MLP(64,32)   &  0.0001 & 0.01   & 128 & 200 & $\eta$=0.001  \\
			Adult      & {\worstoffdro}             &  MLP(64,32)   &  0.00001& 0.001   & 128 & 200 & $\eta$=0.0001  \\
			CelebA     & ERM                     &  ResNet50     &  0.0001 & 0.01   & 128 & 50 & - \\
			CelebA     & {\unsupdro}                &  ResNet50     &  0.0001 & 0.01   & 128 & 50 & $\eta$=0.6  \\
			CelebA     & {\groupdro}-(Oracle)       &  ResNet50     &  0.00001 & 0.1   & 128 & 50 & $\eta$=0.1  \\
			CelebA     & {\groupdro}-(Partial)      &  ResNet50     &  0.00001 & 0.01   & 128 & 50 & $\eta$=0.1  \\
			CelebA     & {\worstoffdro}             &  ResNet50)    &  0.00001 & 0.1   & 128 & 50 & $\eta$=0.001  \\	
			\bottomrule
		\end{tabular}%
	}
	\caption{\footnotesize \label{tab:hparams} \textbf{Hyperparamter choices for Table~\ref{tab:quant_results}.} We list the hyper-parameters selected using the NVP procedure (see Section~\ref{sec:exp}) after performing grid-search. Learning rate and weight decay are an important set of parameters that influences the minority group performance. Each baseline has it's algorithm-specific hyper-parameter such as step-size of the simplex weights in {\groupdro} ($\eta_\text{GDRO}$), the loss threshold in {\unsupdro} ($\eta_\text{UDRO}$) and the step size for the group weights in {\worstoffdro} ($\eta_\text{WDRO}$). The symbol ``-" for batchsize in CMNIST experiments indicate the use of full-batch data for training. }
\end{table}

\subsection{Notes on Optimization}
\label{sec:opti_notes}
When using CVXPY to solve for the {\worstoffdro} assignments, we simplify the problem by replacing the data marginal distribution $\sum_{i=1}^{N}\hat{g}_{ij}$ in the denominator of \eqref{eq:worstoffdro_constraint_set_with_marginal} to $\bar{\mathbf{p}}_j$, thus providing us with a convex optimization problem. The code for the solver is available in Algorithm~\ref{alg:code}. \\

\begin{algorithm}[t]                                                                                    
\caption{Group Assignment Solver using CVXPY library}                                                               
\label{alg:code}
\definecolor{codeblue}{rgb}{0.25,0.5,0.5}
\definecolor{codekw}{rgb}{0.85, 0.18, 0.50}
\lstset{
  backgroundcolor=\color{white},
  basicstyle=\fontsize{7.5pt}{7.5pt}\ttfamily\selectfont,
  columns=fullflexible,
  breaklines=true,
  captionpos=b,
  commentstyle=\fontsize{7.5pt}{7.5pt}\color{codeblue},
  keywordstyle=\fontsize{7.5pt}{7.5pt}\color{codekw},
}
\begin{lstlisting}[language=python]
import cvxpy as cp
import numpy as np

class Solver(object):
  def __init__(self, n_controls, bsize, marginals, epsilon, labeled=None):
    """Group assignment solver.
    
    Arguments:
      n_controls: An integer for the number of groups.
      bsize: An integer for the batch size.
      marginals: A 2D array for the marginal distribution.
      epsilon: A float for the variance.
      labeled: A tuple for labeled data indices and their value.
    """
    self.X = cp.Variable((bsize, n_controls))
    self.l = cp.Parameter((bsize, 1))
    self.p = cp.Parameter((n_controls, 1), value=marginals)
    self.q = cp.Parameter(n_controls)
    if labeled is not None:
      labeled_idx, labeled_value = labeled
    counts = cp.sum(self.X, axis=0, keepdims=True)

    obj = ((self.l.T @ self.X) / self.p.T) @ self.q
    constraints = [self.X >= 0,
                        cp.sum(self.X, axis=1, keepdims=True) == np.ones((bsize, 1)),
                        cp.abs(cp.sum(self.X, axis=0, keepdims=True) / bsize - self.p.T) <= epsilon]
    if labeled is not None:
      constraints += [self.X[labeled_idx] == labeled_value]
      
    self.prob = cp.Problem(cp.Maximize(obj), constraints)

  def cvxsolve(self, losses, weights):
    """Solver.
    
    Arguments:
        losses: A 2D array for loss values.
        weights: A 1D array for group weights q.

    Returns:
      A 2D array for soft group assignments.
    """
    self.l.value = losses
    self.q.value = weights
    self.prob.solve()
    return self.X.value
    
\end{lstlisting}
\end{algorithm}

\subsection{An example of worst-off assignments}
\label{sec:worstoffdro_example}
Using three samples, we provide an example of the worst-off assignments made by our algorithm,
\begin{example}
Consider three samples with loss values $l_1 > l_2 > l_3$ and two predefined groups. Assume the marginal probabilities $\bar{\mathbf{p}}_1=0.6$ and $\bar{\mathbf{p}}_2=0.4$. Without loss in generality, assume $\frac{q_1}{\bar{\mathbf{p}}_1} > \frac{q_2}{\bar{\mathbf{p}}_2}$. With constraint $\mathcal{C}_{\mathbf{\bar p},\epsilon=0}$ and solving for {\worstoffdro} objective results in the following group assignments, $\tiny \{\hat{g}^{t}\} =\begin{pmatrix}
1 & 0\\
0.8 & 0.2\\
0 & 1
\end{pmatrix}$. Here, the $i^{\text{th}}$ row indicates the assignment given to sample $l_{i}$.
\end{example}
The group assignments can be derived by identifying a $\{\hat g\}$ that satisfies the constraints $\sum_{i=1}^N {\hat g}_{i 1} \le N\bar{\mathbf{p}}_1$ and $\sum_{i=1}^N {\hat g}_{i 2} \le N\bar{\mathbf{p}}_2$, where $N=3, \bar{\mathbf{p}}_1=0.6$ and $\bar{\mathbf{p}}_1=0.4$, and correspondingly maximizes {\worstoffdro} objective.
The above example informs us that group assignments depend on the magnitude of loss values in addition to the group weights and marginal probabilities. As indicated in the paper, we find that \textit{high loss samples are assigned to groups with high group weights and low marginal probabilities}, characteristic of a worst-off group.

Marginal constraints form a key ingredient of our algorithm as per the above example. Without the marginal constraints, the group assignments  $\tiny \{\hat{g}^{t}\} =\begin{pmatrix}
1 & 0\\
1 & 0\\
1 & 0
\end{pmatrix}$. That is, the assignments would have been made independent of the loss values and sparsely restricted to the group with large $\frac{q_j}{\bar{\mathbf{p}}_j}$ value.

\subsection{Discussion on Unsupervised DRO methods}
\label{sec:unsupdro_details}
In this section, we contrast {\worstoffdro} method against {\unsupdro} \cite{hashimoto2018fairness} and CVaR DRO \cite{levy2020large}. CVaR DRO \cite{levy2020large} is a coherent risk measure \cite{rahimian2019distributionally} that optimizes over a certain fixed-sized sub-populations within the training dataset. In essence, CVaR DRO is alike {\unsupdro} where the size of the sub-population is controlled by a threshold  on the loss value. In both CVaR DRO and {\unsupdro}, the size of the selected sub-population needs to be close to the size of the smallest group as identified in Section 3.2.2 of \cite{liu2021just}. Such a requirement demands wider hyper-parameter search space for  $\alpha / \eta$ parameters that control the size of the sub-populations. Our experiments justify this need, Table~\ref{tab:hparams_more} of Appendix~\ref{sec:app_hparam_tuning} shows that the search space of {\unsupdro} is twice relative to {\worstoffdro} in order to attain comparable average group accuracies. Clearly, a wider search space contributes to a harder model selection procedure. Moreover, scenarios where extensive search is not possible (eg, small validation set/dataset regimes) could result in incorrect/unstable model selection. From the perspective of the methodology, CVaR DRO / {\unsupdro} train only on the highest loss samples while discarding the remaining samples. In contrast, {\worstoffdro} does not discard any sample rather downweights/upweights as per the worst-off group assignment. This property aids in maintaining a high overall accuracy besides reaching good minority group accuracy.

\subsection{Discussion on MAR case}
The $\delta$ gap in Lemma~\ref{lemma:proof_2} captures the error in misspecification of $\bar{\mathbf{p}}$ in relation to $\mathbf{p}^{\star}$. When $\bar{\mathbf{p}}$ is misspecified due to the data being Missing at Random (MAR) rather than MCAR (Missing Completely at Random), a solution could be to estimate the propensity of missingness from other features; then use inverse propensity weighting to get a consistent estimate of the fraction of samples in each group as discussed in \cite{NEURIPS2020_invprop_udell}. Alternatively, if provided with the knowledge of the data-generation process, the core effort in extending our method simply involves using off-the-shelf estimators to characterize the probability distributions (see \cite{mohan2014graphical} for example.

\subsection{Hyper-parameter Tuning}
\label{sec:app_hparam_tuning}
Hyper-parameters were selected for each algorithm by performing an NVP procedure (see Section~\ref{sec:exp}). The best performing model was identified on the validation set associated with each dataset. All the measures were computed and averaged over three random runs. A list of all the hyper-parameters that were tuned for are available in Table~\ref{tab:hparams_more}. The final hyper-parameters selected for each method can be viewed from Table~\ref{tab:hparams}. 

\subsection{Additional Experimental Results}\label{sec:app_add_result}
We provide the following additional results, first, in Figure~\ref{fig:labpercent_average}, we show average group accuracies as a function of labelled sample counts. The average group accuracies of the {\worstoffdro} method are closely similar across various labelled sample counts. The {\groupdro} method shows a slight increasing trend in the average accuracies as the number of labelled samples increase. Next, corresponding to the quantitative results of Table~\ref{tab:quant_results} in the paper, we provide standard deviations of those results in Table~\ref{tab:quant_results_std}. The standard deviations for all the methods are comparable. 

\subsection{More details on the datasets}
\label{sec:app_more_details}
\subsubsection{Waterbirds}
This dataset was first introduced in \cite{sagawa2019distributionally} and has been developed by cropping images of birds from the CUB dataset \citep{wah2011caltech} and pasting them on the backgrounds from the Places dataset \citep{zhou2017places}. A ResNet50 model, pre-trained with ImageNet weights, has been used for training in experiments on this dataset. No data augmentation has been applied for any of the Algorithms.   

\subsubsection{CMNIST}
CMNIST dataset comprised of two groups of MNIST images each with a specific color. As per the description in the main paper, the target label is flipped with a specific correlation to the color. Following the implementation of \cite{creager2021environment}, the digit images contain two channels and were downsampled to $14\times14$ pixels. 

\subsubsection{Adult}
The Adult dataset used in the paper was obtained from the UCI repository \citep{dua2017uci}. It contains $44,842$ samples. The features that were used in the experiments include ``age", ``workclass", ``fnlwgt", ``education",                           ``education-num", ``marital-status", ``occupation",
                          ``relationship", ``race", ``sex", ``capital-gain",
                           ``capital-loss", ``hours-per-week", ``native-country",
                           ``income". A positve target label in this dataset is indicated by the attribute ``income-bracket" being above $50K\$$.
        
\subsubsection{CelebA}                   
For this dataset, the official train-val-test splits as recommended by \cite{liu2015deep} has been used. Similar to the Waterbirds experiments, a pre-trained ImageNet-based ResNet50 model has been used for the implementations.

\begin{table*}[t] 
    \centering
	\footnotesize
	\begin{tabular*}{0.95\textwidth}{l @{\extracolsep{\fill}}
			*{16}{S[table-format=3.0]}}
		\toprule
		& \multicolumn{2}{c}{Waterbirds} & \multicolumn{2}{c}{CMNIST} & \multicolumn{2}{c}{Adult} & \multicolumn{2}{c}{CelebA}\\
		\cmidrule{2-3} \cmidrule{4-5} \cmidrule{6-7} \cmidrule{8-9} 
		& {\cellcolor{gray!25}\textbf{min}}  & {\textbf{avg}} & {\cellcolor{gray!25}\textbf{min}} & {\textbf{avg}} & {\cellcolor{gray!25}\textbf{min}} &  {\textbf{avg}} & {\cellcolor{gray!25}\textbf{min}} &  {\textbf{avg}}\\ 
		\midrule\midrule
		{\groupdro} (Oracle)   & \cellcolor{gray!25}{$0.45$} &  {$0.02$}  &  {\cellcolor{gray!25}$0.57$}  &  {$0.33$}  &  {\cellcolor{gray!25}$0.94$}  &  {$0.87$}  & {\cellcolor{gray!25}$1.31$}  &  {$0.27$}  \\ \midrule
		ERM                    & {\cellcolor{gray!25}$0.99$} &  {$0.07$}  &  {\cellcolor{gray!25}$1.47$}  &  {$0.49$}  &  {\cellcolor{gray!25}$1.06$}  &  {$0.28$}  & {\cellcolor{gray!25}$3.02$}  &  {$0.05$}  \\
		{\unsupdro}            & {\cellcolor{gray!25}$0.89$} &  {$0.12$}  &  {\cellcolor{gray!25}$1.39$}  &  {$0.64$}  &  {\cellcolor{gray!25}$1.73$}  &  {$0.18$}  & {\cellcolor{gray!25}$1.81$}  &  {$0.02$}  \\
		{\groupdro} (Partial)  & {\cellcolor{gray!25}$1.25$} &  {$1.25$}  &  {\cellcolor{gray!25}$0.76$}  &  {$0.20$}  &  {\cellcolor{gray!25}$0.32$}  &  {$0.34$}  &   {\cellcolor{gray!25}$4.54$}  &  {$0.08$} \\ 
		{\worstoffdro}         & {\cellcolor{gray!25}$0.96$} &  {$0.17$}  &  {\cellcolor{gray!25}{$1.14$}}  &  {$0.35$}  &  {\cellcolor{gray!25}$0.21$}  &  {$0.13$}  &   {\cellcolor{gray!25}$2.05$}  &  {$0.10$}    \\ 	
		\bottomrule
	\end{tabular*} 
	\vspace{-0.05in}
	\caption{ \label{tab:quant_results_std} \small \textbf{Quantitative Results - Standard Deviations.} The standard deviations over three random runs of Table~\ref{tab:quant_results} is provided. For baselines, we consider an ERM, {\unsupdro} \citep{hashimoto2018fairness}, {\groupdro} (Partial) for partly labelled {\groupdro} \citep{sagawa2019distributionally} method, {\groupdro} (Oracle) for the fully supervised model. } 
\end{table*} 

\subsection{Ablation study on increasing the constraint set size.}
\label{sec:increase_set}
We conduct experiments on {\worstoffdro} method for different values of the $\epsilon$ parameter in the set $\{0, 0.001, 0.01, 0.1, 1\}$. The test set accuracies on the minority group and average group are reported in Figure~\ref{fig:epsilon_exps}. Increasing the $\epsilon$ value also increases the constraint set size because the marginal constraint is gradually relaxed. Figure~\ref{fig:epsilon_exps} shows that the both minority group accuracy and average group accuracy values reduce with increase in $\epsilon$ value beyond $0.1$ threshold. The accuracy values for $\epsilon \le 0.01$ are comparable. A similar trend hold on other datasets as well.

\begin{figure*}[t!]
	\begin{subfigure}[b]{0.4\linewidth}
	\includegraphics[width=\linewidth]{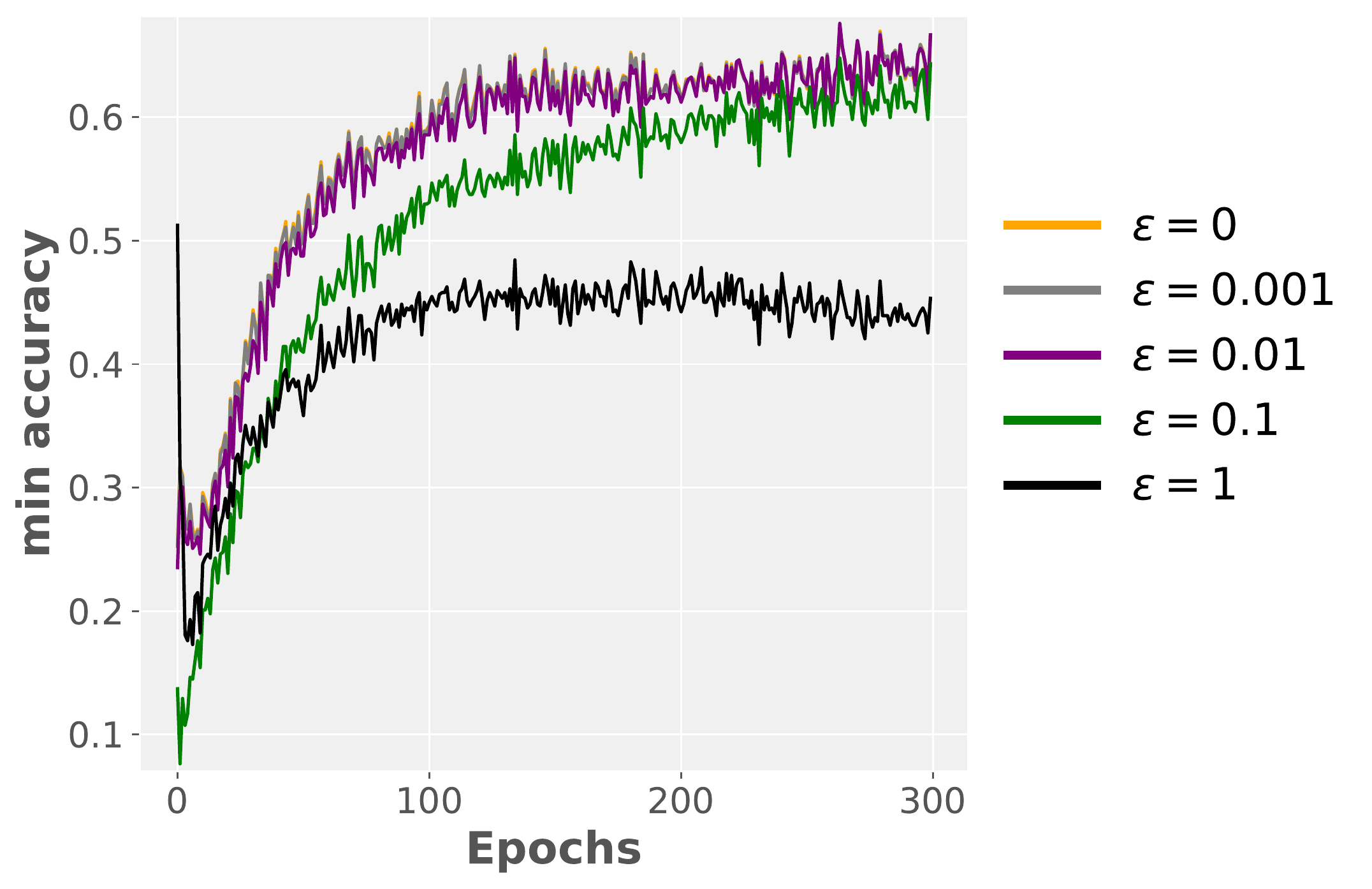}
	\caption{Waterbirds Minority Group Accuracy}
	\end{subfigure} \hspace{2cm}
	\begin{subfigure}[b]{0.4\linewidth}
	\includegraphics[width=\linewidth]{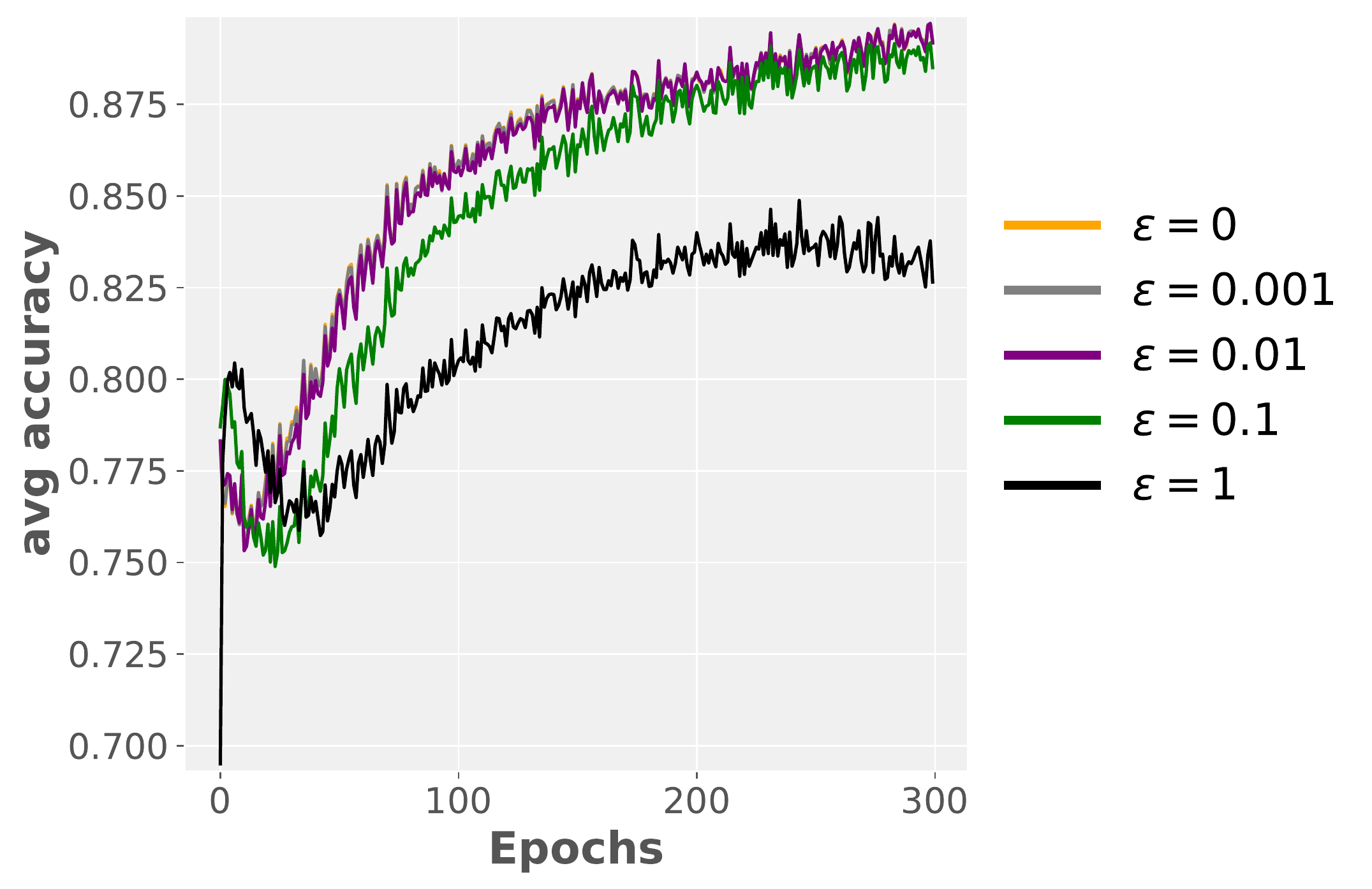}
	\caption{Waterbirds Average Group Accuracy}
	\end{subfigure}
	\centering
	\caption{ \textbf{Varying the $\epsilon$ parameter in the constraints.} The marginal constraint is gradually relaxed by increasing the $\epsilon$ parameter. The accuracies in the plots are computed on the test sets. Performance of the models with $\epsilon \le 0.01$ are similar, however, the accuracies drop when increasing $\epsilon$ beyond $0.01$ threshold.}
	\label{fig:epsilon_exps}
\end{figure*}

\end{document}